\documentclass[11pt]{article}

\usepackage[final]{acl}

\usepackage{times}
\usepackage{latexsym}

\usepackage[T1]{fontenc}

\usepackage{xspace}
\newif\ifshowcomments
\showcommentstrue

\makeatletter
\let\todo\@undefined
\makeatother

\ifshowcomments
    \usepackage[textsize=scriptsize,textwidth=2.4cm]{todonotes}
\else
    \usepackage[disable,textsize=scriptsize,textwidth=2.4cm]{todonotes}
\fi

\usepackage{lineno}
\usepackage{amsmath}
\usepackage{amssymb}
\usepackage{mathtools}
\usepackage{amsthm}
 \usepackage{adjustbox}
\usepackage{booktabs}
\usepackage{multirow}
\usepackage{makecell}
\usepackage{graphicx}
\usepackage{hyperref}
\usepackage{url}
\usepackage{hyperref}
\usepackage{url}
\usepackage{booktabs, multirow, array}
\usepackage{xcolor}
\usepackage{colortbl}
\usepackage{graphicx}
\usepackage{booktabs}
\usepackage{tabularx}
\usepackage{multirow}
\usepackage{threeparttable}
\usepackage{array}
\usepackage{booktabs}
\usepackage{multirow}
\usepackage{makecell}
\usepackage{comment}
\usepackage{makecell}
\usepackage{placeins}
\usepackage{listings}
\usepackage{xcolor}
\usepackage{enumitem}
\usepackage{wrapfig}
\usepackage{listings}
\usepackage{amsthm}
\usepackage{amsmath}

\newtheorem{theorem}{Theorem}
\newtheorem{proposition}{Proposition}
\newtheorem{corollary}{Corollary}
\lstdefinestyle{pythonstyle}{
    language=Python,
    basicstyle=\ttfamily\footnotesize,
    breaklines=true,
    showstringspaces=false
}
\definecolor{darkblue}{rgb}{0, 0, 0.5}
\hypersetup{colorlinks=true, citecolor=darkblue, linkcolor=darkblue, urlcolor=darkblue}
\usepackage{microtype}
\usepackage{hyperref}
\usepackage{url}
\usepackage{booktabs}
\usepackage{graphicx}
\usepackage{booktabs}
\usepackage{subcaption}
\usepackage{multirow}
\usepackage{booktabs}
\usepackage{booktabs}
\usepackage{multirow}
\usepackage{listings}
\usepackage{xcolor}
\usepackage{float}
\usepackage{graphicx} 

\usepackage{makecell}

\usepackage[utf8]{inputenc}
\usepackage{todonotes}

\usepackage{microtype}

\usepackage{inconsolata}

\usepackage{graphicx}

\usepackage{titletoc}

\titlecontents{section}[2.3em]
  {\addvspace{0.5pc}\bfseries} 
  {\contentslabel{2.3em}}    
  {\hspace*{-2.3em}}         
  {} 

\titlecontents{subsection}[3.5em] 
  {\addvspace{0.1pc}}            
  {\contentslabel{2.3em}}    
  {\hspace*{-2.3em}}         
  {} 

\title{Reduced Matrix Multiplication: \\Input-Adaptive Matrix-Product Reduction for LLM Inference}

\author{
  Zixuan Lan\\
  University of Chicago\\
  {\tt\small zixuanlan@uchicago.edu}
  \And
  Yanhong Li\\
  Independent Researcher\\
  {\tt\small yanhong.lbh@gmail.com}
  \And
  Jiawei Zhou\\
  Stony Brook University\\
  {\tt\small jiawei.zhou.1@stonybrook.edu}
}

\newcommand{\name}{RMM}
\begin{document}
\maketitle
\begin{abstract}
Transformer-based language models achieve strong performance but incur substantial inference cost due to repeated high-dimensional matrix multiplications. We propose Reduced Matrix Multiplication (RMM), a training-free, input-adaptive inference method that reduces Transformer matrix products by selecting informative slices along their contraction dimensions, without modifying model weights. Under a simple retention-ratio control, RMM provides a smooth and predictable accuracy–efficiency trade-off.Across language models ranging from 1B to 70B parameters, we find that reduction tolerance depends on the model family, task, component, and retention ratio, although it often improves with model scale. Under moderate reduction, RMM remains robust across the evaluated discriminative, autoregressive-generation, and long-context settings. We further show that the same principle extends to multimodal vision–language inference. Mechanistic ablations reveal a structural asymmetry within Transformers: attention-side computations are substantially more reducible than MLP components. Finally, wall-clock benchmarks with custom kernels on an NVIDIA A100 show that these computational savings can translate into practical runtime gains, especially at longer sequence lengths. Together, these results position RMM as a scalable direction for input-adaptive inference-time optimization. Code and project resources will be available at
\url{https://github.com/Zesearch/rmm-llm}.

\end{abstract}

\section{Introduction}
\label{sec:introduction}

Transformer models~\citep{vaswani2017attention} continue to improve with scale~\citep{kaplan2020scalinglawsneurallanguage}, but their inference cost also grows rapidly. As model size increases, efficient inference becomes increasingly important in practice. A major source of this cost is the repeated high-dimensional matrix products in attention and feed-forward layers. This raises a basic question: during inference, do all indices along the shared multiplication axes of these matrix products need to be evaluated for every input, or can part of this computation be reduced adaptively while preserving model behavior?

Prior work has explored redundant computation in Transformer inference from several angles~\citep{Liu_2021, peng2023structuredpruningselfsupervisedpretrained, Sajjad_2023, liu2023winner}. 
One line of work simplifies model computation through structured pruning, low-rank approximation, or dimension reduction~\citep{sun2024simpleeffectivepruningapproach, ma2023llmprunerstructuralpruninglarge, frantar2023sparsegptmassivelanguagemodels, ashkboos2024slicegptcompresslargelanguage, gao2024displlmdimensionindependentstructuralpruning}. 
Another line of work reduces context-level redundancy through token compression, KV-cache management, or decoding-time scheduling~\citep{li2023compressing, pan2024llmlingua2datadistillationefficient, zhang2023h2o, xiaoefficient, fu2025deepthinkconfidence, shah2024flashattention3fastaccurateattention, yuan2025nativesparseattentionhardwarealigned, ICLR2025_0d4d9fc3, li2025text}. 
These methods reduce inference cost by modifying fixed model structures, shortening inputs or caches, or changing execution flow, but they do not directly ask whether the contracted computation inside each Transformer matrix product can be reduced adaptively for the current input. 
A closely related direction exploits activation sparsity by skipping small-magnitude activation entries during inference~\citep{liu2025trainingfreeactivationsparsitylarge, lee2024catscontextuallyawarethresholdingsparsity}. 
Although such methods also use input-dependent activation information, their primary object is the sparsification of activation tensors or hidden states. 
Our focus is different: rather than sparsifying hidden states themselves, we reduce the shared multiplication axis of each matrix product. 
Depending on the operation, this axis may correspond to hidden channels in linear or MLP projections, attention-head feature dimensions in attention-score computation, or token positions in the attention-value product. 
This framing distinguishes {\name} from fixed component pruning, activation-state sparsification, and token pruning: the goal is to reduce the contracted computation performed by each matrix product at inference time.

In this work, we propose \emph{Reduced Matrix Multiplication} ({\name}), a training-free, input-adaptive method for Transformer inference. Rather than executing the full matrix products in attention and MLP layers, {\name} dynamically selects and computes informative indices along the shared multiplication axis of each product, without modifying model weights. Beyond acceleration, {\name} provides a controllable way to reduce computation through a simple retention ratio, enabling a systematic study of redundancy in Transformer inference. Across models ranging from 1B to 70B parameters and diverse downstream tasks, we observe a general trend that larger models tolerate more aggressive reduction, although reduction robustness remains model- and task-dependent. We further observe a clear structural asymmetry within Transformers: attention-side computations are substantially more reducible, whereas MLP components are much more sensitive to reduction. We also show that {\name} extends to multimodal vision--language inference, and implement Triton kernels to verify that the reduced matrix products can translate into practical wall-clock speedups.

\section{Related Work}

Prior work has shown that Transformer inference contains substantial redundancy across neurons, layers, attention heads, and larger structural units: models can often tolerate pruning or reduction while preserving performance~\citep{peng2023structuredpruningselfsupervisedpretrained, Sajjad_2023}. These findings suggest that not all inference-time computation is equally indispensable.

One major direction reduces computation by modifying or compressing model structure, including structured pruning, unstructured pruning, low-rank approximation, and dimension reduction~\citep{sun2024simpleeffectivepruningapproach, ma2023llmprunerstructuralpruninglarge, frantar2023sparsegptmassivelanguagemodels, ashkboos2024slicegptcompresslargelanguage, gao2024displlmdimensionindependentstructuralpruning}. Such methods typically determine the reduction pattern before inference and operate on a fixed model structure, answering which parts of the model can be removed or compressed in advance rather than whether computation should vary with the current input.

Another direction reduces redundancy in the input context, cache, attention pattern, or execution process through token compression, KV-cache management, decoding-time skipping, sparse attention, efficient attention kernels, or scheduling optimization~\citep{li2023compressing, pan2024llmlingua2datadistillationefficient, zhang2023h2o, xiaoefficient, shah2024flashattention3fastaccurateattention, yuan2025nativesparseattentionhardwarealigned}. These methods can reduce latency or memory overhead, especially in long-context settings, but they mainly act on inputs, cached tokens, attention patterns, or execution flow rather than directly reducing the high-dimensional matrix products inside attention and feed-forward layers.

Closely related to our work are training-free activation-sparsity methods, including TEAL~\citep{liu2025trainingfreeactivationsparsitylarge} and CATS~\citep{lee2024catscontextuallyawarethresholdingsparsity}, which skip or mask low-magnitude activation entries during inference. These methods are also input-dependent and often training-free. For linear and MLP projections, {\name} shares a related mechanism with activation sparsity, since both use the current activations to retain a subset of input dimensions. The main distinction lies in formulation and scope: activation-sparsity methods define sparsity over activation entries supplied to projection layers, whereas {\name} defines reduction over the shared contraction axis of a general matrix product. This matrix-product formulation applies not only to linear and MLP projections, but also to attention-internal products such as $QK^\top$ and $PV$, where the contracted axes correspond to attention-head feature dimensions and token positions, respectively. Thus, {\name} overlaps with activation sparsity in the projection setting while extending the same reduction principle to a broader set of Transformer matrix products.

Our work is also related to classical randomized approximate matrix multiplication~\citep{drineas2006fast}, which approximates a matrix product by sampling column-row pairs according to importance distributions. Such estimators can in principle be used within a single Transformer forward pass. However, under a fixed reduction budget, stochastic sampling introduces per-instance approximation variance; reducing this variance generally requires more sampled pairs, which weakens the achievable computational savings. {\name} adopts the same matrix-product perspective while using deterministic, activation-aware index selection along the shared multiplication axis, yielding predictable reduced products whose effects we evaluate across Transformer components. Overall, {\name} takes a direct computational perspective: it reduces matrix products themselves.

\section{Methodology}
\label{sec:method}

\subsection{Preliminaries}

\begin{figure*}[t]
    \centering
    \includegraphics[width=\linewidth]{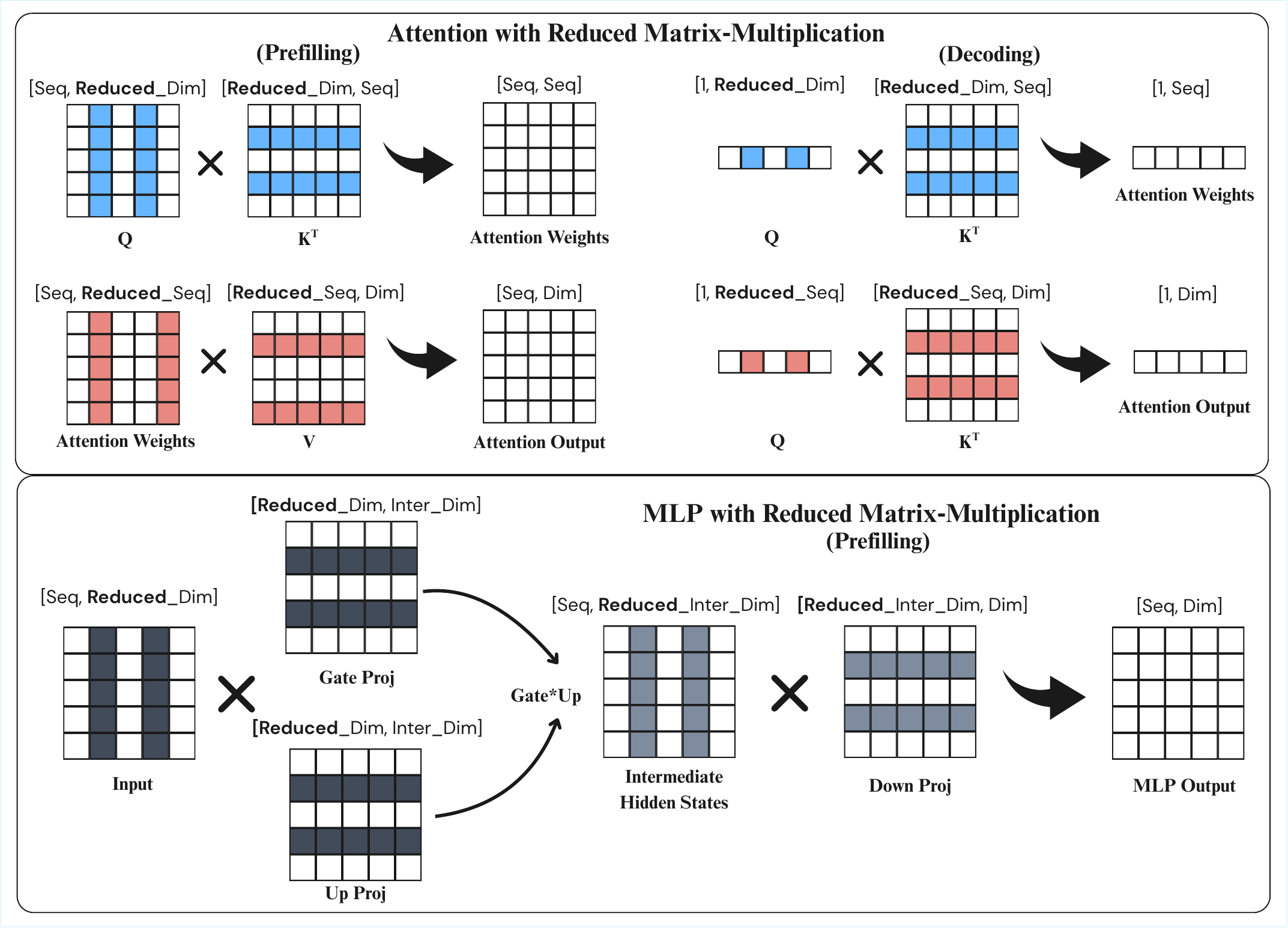}
     \vspace{-2.0em}
    \caption{Application of {\name} in major computations of Transformer language models.}
    \label{fig:Method}
\end{figure*}

Transformer inference is mainly composed of matrix multiplications in self-attention and feed-forward networks (MLPs). Let the input hidden states at layer \(l\) be \(X^{(l)} \in \mathbb{R}^{L \times d}\), where \(L\) is the sequence length and \(d\) is the hidden dimension. In self-attention, the input is linearly projected to queries, keys, and values as \(Q = X^{(l)}W_Q\), \(K = X^{(l)}W_K\), and \(V = X^{(l)}W_V\), where \(W_Q, W_K, W_V \in \mathbb{R}^{d \times d_h}\), and \(d_h\) is the feature dimension of each attention head. Attention scores are computed by \(QK^\top\), and the output is obtained by \(\mathrm{softmax}(QK^\top)V\). Similarly, the MLP block consists of large linear transformations between activations and weight matrices. Although these computations arise in different modules, their core can be written in the unified form \(Y = AB\), where \(A \in \mathbb{R}^{n \times d}\) denotes the activation matrix determined by the current input, \(B \in \mathbb{R}^{d \times m}\) denotes a weight matrix or intermediate representation, and \(Y \in \mathbb{R}^{n \times m}\) is the output.
\subsection{Reduced Matrix Multiplication}
\label{sec:rmm_problem}

Classical approximate matrix multiplication establishes that the product $AB$ can be approximated via Monte Carlo sampling: one repeatedly draws column-row pairs from the shared dimension according to a probability distribution derived from both matrices, rescales them, and averages over independent draws to obtain a low-error estimate in expectation~\citep{drineas2006fast}. However, Transformer inference performs only a single forward pass per input, leaving no opportunity to average over repeated samples, so such randomized estimators cannot be directly applied with reliable per-instance quality. This motivates a deterministic, input-adaptive approach. Based on the unified matrix multiplication form above, we define Reduced Matrix Multiplication({\name}) as follows. For a matrix product \(Y = AB\), where \(A \in \mathbb{R}^{n \times d}\) and \(B \in \mathbb{R}^{d \times m}\), {\name} selects an index set \(\mathcal{I} \subseteq [d]\) with \(|\mathcal{I}|=\lceil \rho d\rceil\), where \(\rho\in(0,1]\) is a user-controlled retention ratio, and computes
\[
\mathrm{RMM}_\rho(A,B) \triangleq A_{:,\mathcal{I}}\, B_{\mathcal{I},:}.
\]

\paragraph{Activation-aware dimension selection.}
To select \(\mathcal{I}\), we assign each feature dimension \(j\in[d]\) an importance score \(s_j \triangleq \|A_{:,j}\|_2\), which measures the magnitude of feature \(j\) under the current input. Given a retention ratio \(\rho\), we select \(\mathcal{I}=\operatorname{TopK}(\{s_j\}_{j=1}^d,\lceil \rho d\rceil)\). This procedure is fully deterministic. Since the selection depends on the current activations, the resulting subspace may vary across inputs, layers, attention heads, and decoding steps. {\name} adapts its computation to each input by selecting dimensions based on current activation magnitudes. This choice is theoretically grounded: we prove in Appendix~\ref{app:minimax} that TopK selection by column norm is minimax optimal, minimizing the worst-case approximation error over all possible $B$ at any given retention budget. Our ablation experiments (Section~\ref{sec:ablation and me}) further confirm that both dynamic selection and activation-aware scoring are essential to the effectiveness of {\name}.

\subsection{Applying RMM to Attention and MLPs}
\label{sec:rmm_apply}

We apply {\name} to Attention and MLP layers as follows (Figure~\ref{fig:Method}).
\paragraph{Attention.}
Consider a single attention head with queries \(Q\in\mathbb{R}^{L_q\times d_h}\), keys \(K\in\mathbb{R}^{L_k\times d_h}\), and values \(V\in\mathbb{R}^{L_k\times d_h}\). We compute feature scores \(s_j=\|Q_{:,j}\|_2\) and select \(\mathcal{I}=\operatorname{TopK}(\{s_j\}_{j=1}^{d_h},\lceil\rho_d d_h\rceil)\), yielding reduced attention scores $\widetilde{S}=\frac{1}{\sqrt{d_h}}\,Q_{:,\mathcal{I}}K_{:,\mathcal{I}}^\top$.
Attention weights are then obtained as \(P=\mathrm{softmax}(\widetilde{S}+M)\), where \(M\) denotes the optional causal or attention mask. For grouped-query attention, dimension selection is performed per head on \(Q\), with the corresponding dimensions gathered from the shared \(K\) and \(V\) tensors. We further optionally sparsify the attention--value multiplication \(PV\) over the token dimension by computing token scores \(a_t=\|P_{:,t}\|_2\), selecting \(\mathcal{T}=\operatorname{TopK}(\{a_t\}_{t=1}^{L_k},\lceil\rho_t L_k\rceil)\), and evaluating \(\widetilde{O}=P_{:,\mathcal{T}}V_{\mathcal{T},:}\).

\paragraph{MLP and linear projections.}
Given activations \(X\in\mathbb{R}^{L\times d}\) and weights \(W\in\mathbb{R}^{d\times d'}\), we compute feature scores \(s_j=\|X_{:,j}\|_2\), select \(\mathcal{I}=\operatorname{TopK}(\{s_j\}_{j=1}^{d},\lceil\rho_d d\rceil)\), and evaluate \(\widetilde{Y}=X_{:,\mathcal{I}}W_{\mathcal{I},:}\). The same rule applies to linear projections as well as to feed-forward layers.

\subsection{Complexity}
\label{sec:rmm_complexity}

For a matrix multiplication \(A\in\mathbb{R}^{n\times d}\) and \(B\in\mathbb{R}^{d\times m}\), dense computation costs \(O(ndm)\). With feature retention ratio \(\rho_d\), {\name} evaluates \(A_{:,\mathcal{I}}B_{\mathcal{I},:}\) with \(|\mathcal{I}|=\lceil \rho_d d\rceil\), reducing the cost to \(O(n\rho_d d\,m)\). In attention, reducing \(QK^\top\) over the head dimension lowers the cost from \(O(L_qL_kd_h)\) to \(O(L_qL_k\rho_d d_h)\). If token selection is also applied to the attention--value product, the cost of \(PV\) is reduced from \(O(L_qL_kd_h)\) to \(O(L_q\rho_tL_kd_h)\). {\name} additionally requires computing feature scores and selecting top-\(k\) indices: computing \(s_j=\|A_{:,j}\|_2\) costs \(O(nd)\), while top-\(k\) selection over \(d\) (or \(L_k\)) is a lightweight vector-level operation. In practice, these overheads are small relative to the dense matrix multiplications that {\name} replaces.

\subsection{Theoretical Justification}
\label{sec:rmm_theory}

We provide a concise theoretical justification for the
activation-aware selection rule used by {\name}. Consider a matrix
product $AB$, where $A\in\mathbb{R}^{n\times d}$ and
$B\in\mathbb{R}^{d\times m}$. The product can be decomposed into
rank-one terms along its shared multiplication axis:
\begin{equation}
AB=\sum_{j=1}^{d}A_{:,j}B_{j,:}.
\end{equation}
For a retained index set $\mathcal{I}$, {\name} discards the terms
indexed by $\bar{\mathcal{I}}=[d]\setminus\mathcal{I}$. The resulting
approximation error satisfies
\begin{equation}
\begin{aligned}
AB-A_{:,\mathcal{I}}B_{\mathcal{I},:}
&=\sum_{j\in\bar{\mathcal{I}}}A_{:,j}B_{j,:},\\
\left\|AB-A_{:,\mathcal{I}}B_{\mathcal{I},:}\right\|_F
&\leq
\sum_{j\in\bar{\mathcal{I}}}
\left\|A_{:,j}\right\|_2
\left\|B_{j,:}\right\|_2.
\end{aligned}
\label{eq:main_error_bound}
\end{equation}

This decomposition motivates selecting dimensions according to the
observed activation matrix. Let
$\alpha_j=\|A_{:,j}\|_2$ and
$b_j=\|B_{j,:}\|_2$. For selection rules that depend only on $A$,
we consider the worst-case value of the upper bound in
Eq.~\eqref{eq:main_error_bound} under the normalized constraint
$\sum_j b_j^2\leq 1$. For any fixed $\mathcal{I}$,
Cauchy--Schwarz gives
\begin{equation}
\max_{\substack{b_j\geq 0\\\sum_j b_j^2\leq 1}}
\sum_{j\in\bar{\mathcal{I}}}\alpha_jb_j
=
\left(
\sum_{j\in\bar{\mathcal{I}}}\alpha_j^2
\right)^{1/2}
=
\left\|A_{:,\bar{\mathcal{I}}}\right\|_F.
\label{eq:main_minimax}
\end{equation}
The equality is attained when the unknown row energies are
proportional to the discarded activation-column norms. Therefore,
under a fixed budget $|\mathcal{I}|=k$, minimizing the worst-case
upper-bound surrogate is equivalent to minimizing the residual
activation energy
$\sum_{j\notin\mathcal{I}}\|A_{:,j}\|_2^2$. The optimal solution is
precisely to retain the $k$ columns of $A$ with the largest
$\ell_2$ norms. Thus, the TopK rule used by {\name} is minimax
optimal for the stated activation-only surrogate.

The same analysis also provides an energy-based interpretation of
the retention ratio. Applying Cauchy--Schwarz to
Eq.~\eqref{eq:main_error_bound} yields
\begin{equation}
\left\|AB-A_{:,\mathcal{I}}B_{\mathcal{I},:}\right\|_F
\leq
\left\|A_{:,\bar{\mathcal{I}}}\right\|_F
\left\|B_{\bar{\mathcal{I}},:}\right\|_F.
\label{eq:main_factorized_bound}
\end{equation}
Defining the discarded energy ratios
\begin{equation}
\epsilon_A(\rho)=
\frac{\|A_{:,\bar{\mathcal{I}}}\|_F^2}{\|A\|_F^2},
\qquad
\epsilon_B(\rho)=
\frac{\|B_{\bar{\mathcal{I}},:}\|_F^2}{\|B\|_F^2},
\end{equation}
we obtain the normalized approximation bound
\begin{equation}
\frac{
\|AB-A_{:,\mathcal{I}}B_{\mathcal{I},:}\|_F
}{
\|A\|_F\|B\|_F
}
\leq
\sqrt{\epsilon_A(\rho)\epsilon_B(\rho)}.
\label{eq:main_normalized_bound}
\end{equation}
TopK selection minimizes the $A$-side residual energy among all
index sets with the same cardinality. Consequently, when activation
energy is concentrated in a subset of contraction dimensions,
{\name} can retain most of the activation energy while evaluating a
smaller matrix product. Increasing $\rho$ monotonically decreases
this residual, providing a direct theoretical connection between the
retention ratio and approximation quality.

This argument applies uniformly to the contraction axes targeted by
{\name}: input channels in linear and MLP projections, head-feature
dimensions in $QK^\top$, and token positions in $PV$. The analysis
justifies the selection rule rather than guaranteeing downstream
task accuracy, which can additionally depend on cancellation,
nonlinearities, and error propagation across layers. The complete
theorem, proofs, and further discussion are provided in
Appendix~\ref{app:theory}.

\section{Experimental Setup}
\label{sec:experiments setup}

\begin{table*}[t]
\centering
\small
\renewcommand{\arraystretch}{1.05}
\begin{tabular*}{\textwidth}{
    @{\extracolsep{\fill}}
    lcccccc
    @{}
}
\toprule
Method (RR $=0.5$)
& ARC-C & ARC-E & COPA & PIQA & CommQA & Avg. \\
\midrule
Full model (RR $=1.0$)
& 49.5 & 76.3 & 77.2 & 79.9 & 66.0 & 69.8 \\

\name{}
& 36.8 & 63.0 & 70.6 & 76.6 & 51.9 & 59.8 \\

SparseGPT
& 31.4 & 64.2 & 70.4 & 71.0 & 43.4 & 56.1 \\

Wanda
& 28.1 & 60.7 & 67.2 & 68.9 & 38.4 & 52.7 \\

SliceGPT
& 20.7 & 31.8 & 56.0 & 53.4 & 23.1 & 37.0 \\

Magnitude
& 22.7 & 33.7 & 57.2 & 57.6 & 25.0 & 39.3 \\
\bottomrule
\end{tabular*}

\vspace{-0.6em}
\caption{
Zero-shot QA performance on LLaMA~3.1~8B under a fixed retention
ratio ($\mathrm{RR}=0.5$). All pruning methods operate at the same
retention ratio, while the full model ($\mathrm{RR}=1.0$) is shown
for reference.
}
\label{tab:llama8b_qa_other_methods}
\end{table*}

\begin{table*}[t]
\vspace{-0.6em}
\centering
\small
\renewcommand{\arraystretch}{1.0}
\begin{tabular*}{\textwidth}{
    @{\extracolsep{\fill}}
    llcccccc
    @{}
}
\toprule
Model & Method & RR
& ROUGE-1 & ROUGE-2 & ROUGE-L & ROUGE-Lsum & BERTScore \\
\midrule
\multirow{9}{*}{LLaMA~3.1~8B}
& Full model & 1.0
& 37.4 & 15.6 & 24.3 & 31.3 & 86.8 \\

& \name{} & 0.8
& 37.5 & 15.7 & 24.2 & 31.4 & 86.7 \\

& \name{} & 0.5
& \textbf{34.2} & \textbf{13.6} & \textbf{22.0}
& \textbf{28.7} & \textbf{85.8} \\

& Static & 0.8
& 37.4 & 15.6 & 24.2 & 31.2 & 86.7 \\

& Static & 0.5
& 28.0 & 9.9 & 19.3 & 24.3 & 84.0 \\

& Random & 0.8
& 6.9 & 0.6 & 6.2 & 6.7 & 78.0 \\

& Random & 0.5
& 5.7 & 0.2 & 5.2 & 5.5 & 81.4 \\

& H2O & 0.8
& 24.4 & 9.3 & 16.3 & 21.9 & 82.7 \\

& H2O & 0.5
& 24.4 & 9.3 & 16.3 & 21.9 & 82.7 \\
\bottomrule
\end{tabular*}

\vspace{-0.8em}
\caption{
Abstractive summarization performance on CNN/DailyMail using
LLaMA~3.1~8B under different pruning strategies. RMM, static
pruning, and random pruning are evaluated at the indicated retention
ratios. H2O uses a fixed token budget and therefore yields identical
results across the two displayed settings.
}
\label{tab:summary_results}
\end{table*}

\paragraph{Overview.}
Our experiments are structured to validate both the effectiveness of dynamic, activation-aware pruning and the empirical insights it enables under controlled retention ratios. We first establish the necessity of dynamic selection by comparing {\name} against representative weight-level static pruning methods under matched sparsity budgets on LLaMA~3.1~8B. We then sweep retention ratios across model scales from 1B to 70B to examine how performance degrades as computation is reduced and how redundancy varies with scale. Next, we test robustness under more realistic inference settings, including autoregressive generation and long-context reasoning. We further perform component-wise pruning analyses on attention and MLP blocks to identify which computations are more redundant and which are more critical. Finally, we evaluate generalization on a vision--language model and report wall-clock latency on an NVIDIA A100 GPU to verify that the computational savings translate into actual runtime improvements. Additional experiments, including comparisons with TEAL, compute-normalized component analysis, evaluations on additional VLM backbones, INT8 compatibility, and LLaMA-70B latency, are provided in Appendix~\ref{sec:supp-results}.

\paragraph{Models and tasks}
We evaluate {\name} on a wide spectrum of pre-trained LLMs, including Llama~3.1~70B, Llama~3.1~8B, Llama~3.2~3B, Llama~3.2~1.5B \citep{grattafiori2024llama3herdmodels}, Qwen3~32B and Qwen~3.1~7B \citep{yang2025qwen3technicalreport}, and Qwen2.5-VL-7B-Instruct \citep{bai2025qwen25vltechnicalreport}.
The benchmarks span multiple capabilities: (i) general QA and reasoning, including \textsc{Copa} \citep{gordon-etal-2012-semeval}, \textsc{PiQA} \citep{bisk2020piqa}, \textsc{CommonsenseQA} \citep{talmor2019commonsenseqaquestionansweringchallenge}, \textsc{ARC-Easy}, \textsc{ARC-Challenge} \citep{clark2018thinksolvedquestionanswering}, and \textsc{MMLU} \citep{hendrycks2021measuringmassivemultitasklanguage}; (ii) language modeling evaluation on \textsc{WikiText} \citep{merity2016pointersentinelmixturemodels} and \textsc{BookCorpus} \citep{zhu2015aligning}; (iii) mathematics and coding tasks, including \textsc{GSM8K} \citep{cobbe2021trainingverifierssolvemath} and \textsc{HumanEval} \citep{chen2021evaluatinglargelanguagemodels}; (iv) long-context reasoning, including \textsc{Ruler-CWE} and \textsc{Ruler-Hotpot} \citep{hsieh2024rulerwhatsrealcontext}; (v) summarization on \textsc{CNN/DailyMail} \citep{nallapati-etal-2016-abstractive}; and (vi) vision--language tasks, including \textsc{POPE} \citep{li2023evaluatingobjecthallucinationlarge}, Blink Art Style, Blink Forensic Detection, and Blink Counting \citep{fu2024blinkmultimodallargelanguage}.
All tasks are evaluated in the zero-shot setting without task-specific fine-tuning. For a more controlled analysis, the main-paper results apply reduction to attention-side matrix multiplications. Full detailed result tables are provided in the appendix.

\paragraph{Baselines.}
We compare \name{} against representative pruning and inference-time optimization baselines. \textbf{Static pruning methods} include SparseGPT~\citep{frantar2023sparsegptmassivelanguagemodels}, Wanda~\citep{sun2024simpleeffectivepruningapproach}, SliceGPT~\citep{ashkboos2024slicegptcompresslargelanguage}, and magnitude pruning~\citep{han2015learningweightsconnectionsefficient}. For \textbf{dynamic inference-time baselines}, we include H2O~\citep{zhang2023h2o}, which dynamically manages the KV cache during decoding but does not modify the feature dimensions involved in matrix multiplications. As control baselines, we also include a \textbf{static variant of \name{}}, which selects a fixed subset of feature dimensions during prefill based on activation statistics and reuses the same subset for all subsequent tokens. And \textbf{random pruning} retains feature dimensions uniformly at random at each decoding step under the same sparsity budget as \name{}. Together, these baselines allow us to compare \name{} against static pruning, non-adaptive activation-based selection, random selection, and dynamic methods operating at different levels of the inference stack.

\section{Main Results}
\label{sec:main_results}

\subsection{Controlled comparison of Pruning behavior}

We first compare pruning strategies on a single model under a controlled setting to isolate the effect of different pruning strategies. All experiments in this stage are conducted on \textbf{LLaMA~3.1~8B} with a fixed retention ratio of $\mathrm{RR}=0.5$. We consider two inference settings:  

\paragraph{Discriminative question answering}

We first evaluate zero-shot QA performance on standard reasoning benchmarks (Table~\ref{tab:llama8b_qa_other_methods}). Performance is measured by comparing the log-likelihoods of candidate answers. We compare \name{} against representative static pruning baselines, including SparseGPT, Wanda, SliceGPT, and magnitude pruning. \name{} achieves the best average accuracy among all pruning methods and shows more consistent degradation across tasks. In contrast, static baselines suffer substantially larger and less uniform drops across benchmarks. These results suggest that fixed, non-adaptive pruning decisions are insufficient to maintain stable performance in practical downstream tasks.

\paragraph{Abstractive summarization}

We further evaluate pruning strategies on the abstractive summarization task, a token-by-token generation setting (Table~\ref{tab:summary_results}). In addition to the static baselines, we include two dynamic baselines in this setting: (i) random pruning, which selects retained dimensions uniformly at random at each decoding step, and (ii) H2O, which dynamically manages the KV cache at the token level. All methods are evaluated under identical retention ratios. At $\mathrm{RR}=0.8$, \name{} remains close to the full model while substantially reducing computation. When the retention ratio is reduced to $\mathrm{RR}=0.5$, \name{} continues to outperform all baselines by a clear margin. Static pruning methods degrade rapidly in generation quality, while random pruning, despite being dynamic, fails to maintain coherent and semantically consistent summaries. H2O performs better than static pruning in this setting, but remains consistently weaker than activation-aware matrix-level pruning.

\paragraph{Summary}

Under a fixed model and retention ratio, the relative behavior of pruning strategies differs markedly between discriminative and generative inference settings. Static and activation-agnostic methods exhibit less stable degradation, while \name{} maintains a clear advantage across both settings. Additional detailed results are reported in Table~\ref{tab:summary_results_full}.

\subsection{Scaling across models and retention ratios}

We next examine how pruning tolerance changes with model scale under different retention ratios. We evaluate {\name} on a range of model sizes. For each model, we sweep the retention ratio over $\mathrm{RR} \in \{0.9, 0.8, 0.7, 0.6, 0.5\}$. We use the same zero-shot discriminative evaluation setup, covering commonsense and reasoning benchmarks as well as more structured tasks including \textsc{GSM8K}, \textsc{MMLU}, and \textsc{HumanEval}.

Table~\ref{tab:main_results} summarizes the results across models and retention ratios. A complementary is provided in Figure~\ref{fig:retention_robustness} in the appendix. At matched retention levels, larger models generally retain stronger performance under moderate reduction, although the trend varies across model families and tasks. For example, at $\mathrm{RR}=0.8$, LLaMA~3.1~70B remains close to the full model on most benchmarks, whereas smaller models show more pronounced degradation, especially on challenging tasks such as \textsc{GSM8K} and \textsc{HumanEval}. As the retention ratio decreases further, all models degrade, but smaller models exhibit an earlier performance inflection---with noticeable drops already at $\mathrm{RR}=0.7$---while larger models degrade more gradually and maintain higher absolute accuracy even under aggressive pruning ($\mathrm{RR}\leq0.6$). Overall, the results suggest a broad scaling trend in which larger models often tolerate stronger reduction, while also showing that robustness remains model- and task-dependent. Additional results about 1B and 3B models are shown in Figure \ref{fig:1B3B} in Appendix~\ref{sec:supp-results}.

\begin{table*}[t]
\centering
\scriptsize
\renewcommand{\arraystretch}{0.85}
\setlength{\tabcolsep}{3pt}
\begin{tabularx}{\textwidth}{l l c *{8}{X}}
\toprule
\textbf{Model} & \textbf{Method} & \text{RR} & \textbf{Copa} & \textbf{ARC-C} & \textbf{ARC-E} & \textbf{PiQA} & \textbf{CommQA} & \textbf{GSM8K} & \textbf{MMLU} & \textbf{HumanEval} \\
\midrule
\multirow{6}{*}{Qwen3.1 7B}
  & Baseline & --   & 72.8 & 39.8  & 69.8 & 72.7 & 47.6 & 39.9 & 55.5 & 40.2 \\
  & {\name}    & 0.9  & 66.0 & 33.8 & 64.6 & 69.4 & 46.7 & 24.9 & 52.7 & 39.6 \\
  & {\name}   & 0.8  & 64.0 & 29.8 & 57.2 & 67.2 & 47.6 & 17.6 & 47.3 & 38.4 \\
  & {\name}    & 0.7  & 60.8 & 28.1 & 46.8 & 63.5 & 39.6 & 5.8  & 33.6 & 31.1 \\
  & {\name}   & 0.6  & 58.4 & 25.4 & 37.0 & 59.6 & 29.9 & 2.5  & 26.0 & 20.7 \\
  & {\name}   & 0.5  & 49.4 & 20.7 & 32.1 & 53.2 & 24.5 & 1.7  & 23.8 & 9.8  \\
\midrule
\multirow{6}{*}{Llama3.1 8B}
  & Baseline & --   & 77.2 & 49.5  & 76.3 & 79.9 & 66.0 & 26.2 & 63.5 & 35.4 \\
  & {\name}   & 0.9  & 76.6 & 48.2 & 75.3 & 79.2 & 65.4 & 24.6 & 62.2 & 35.4 \\
  & {\name}    & 0.8  & 77.2 & 47.5 & 75.1 & 79.1 & 64.7 & 23.7 & 60.3 & 34.8 \\
  & {\name}    & 0.7  & 77.0 & 46.8 & 72.8 & 77.5 & 62.7 & 23.2 & 55.2 & 32.3 \\
  & {\name}   & 0.6  & 73.4 & 37.5 & 68.6 & 77.5 & 59.4 & 14.9 & 38.6 & 26.2 \\
  & {\name}    & 0.5  & 70.6 & 36.8 & 63.0 & 76.7 & 51.9 & 5.9  & 24.8 & 23.2 \\
\midrule
\multirow{6}{*}{Qwen3 32B}
  & Baseline & --   & 81.4 & 57.9 & 78.3 & 80.9 & 61.6 & 62.6 & 80.8 & 37.8 \\
  & {\name}    & 0.9  & 83.6 & 55.2 & 76.1 & 80.7 & 62.2 & 62.7 & 80.0 & 40.2 \\
  & {\name}   & 0.8  & 83.2 & 51.8 & 72.6 & 80.2 & 61.0 & 58.0 & 78.6 & 42.1 \\
  & {\name}    & 0.7  & 82.6 & 48.5 & 69.5 & 80.6 & 60.6 & 55.1 & 77.5 & 42.7 \\
  & {\name}   & 0.6  & 82.6 & 50.8 & 70.5 & 80.4 & 58.7 & 50.9 & 73.2 & 45.1 \\
  & {\name}   & 0.5  & 82.2 & 46.2 & 67.2 & 77.8 & 54.9 & 39.9 & 65.1 & 46.1 \\
\midrule
\multirow{6}{*}{Llama3.1 70B}
  & Baseline & --   & 84.4 & 56.2 & 78.3 & 83.2 & 58.0 & 53.7 & 75.3 & 51.2 \\
  & {\name}    & 0.9  & 84.4 & 54.5 & 78.8 & 83.5 & 58.0 & 51.5 & 75.0 & 53.7 \\
  & {\name}    & 0.8  & 84.6 & 56.9 & 76.8 & 82.6 & 59.7 & 48.1 & 72.6 & 47.0 \\
  & {\name}   & 0.7  & 81.4 & 53.2 & 74.7 & 82.5 & 59.9 & 42.8 & 67.0 & 39.6 \\
  & {\name}    & 0.6  & 76.6 & 50.5 & 74.9 & 77.7 & 60.4 & 34.5 & 53.8 & 34.8 \\
  & {\name}   & 0.5  & 70.2 & 41.5 & 64.2 & 73.6 & 56.8 & 19.9 & 29.7 & 18.9 \\
\bottomrule
\end{tabularx}
\vspace{-0.8em}
\caption{Performance comparison across different models and retention ratios. See
Figure~\ref{fig:retention_robustness} for visualization.}
\label{tab:main_results}
\end{table*}

\subsection{Stability and robustness under generation and long-context settings}

While the previous sections focus on discriminative benchmarks, practical inference also requires stable long-form generation and reliable reasoning over extended contexts. We therefore further evaluate {\name} under both generative and long-context settings to assess whether dynamic pruning remains stable beyond discriminative task evaluation.

\begin{table}[h]
\centering
\footnotesize
\renewcommand{\arraystretch}{0.92}
\setlength{\tabcolsep}{8pt}
\begin{tabular}{llccc}
\toprule
Task & Cond. & 5K & 15K & 30K \\
\midrule
\multirow{3}{*}{CWE}
  & Base & 98.2 & 94.0 & 29.6 \\
  & 0.8  & 98.1 & 94.1 & 29.3 \\
  & 0.5  & 98.0 & 94.0 & 28.9 \\
\midrule
\multirow{3}{*}{Hotpot}
  & Base & 53.6 & 56.4 & 51.2 \\
  & 0.8  & 53.6 & 55.8 & 50.8 \\
  & 0.5  & 53.6 & 55.6 & 50.5 \\
\bottomrule
\end{tabular}
\vspace{-0.8em}
\caption{Long-context results on Ruler.}
\label{tab:ruler_results}
\end{table}
\paragraph{Qualitative generation behavior}
We next examine how pruning affects autoregressive generation. Table~\ref{tab:generation_comparison} shows representative outputs from the full model and {\name} under different retention ratios. At $\mathrm{RR}=0.8$, the outputs of {\name} remain largely consistent with those of the full model in both semantic content and overall structure. At $\mathrm{RR}=0.6$, the generated text begins to show simplification and stylistic drift, but remains coherent and semantically aligned with the prompt. In these examples, stronger reduction introduces simplification and stylistic drift rather than an abrupt loss of coherence.

\begin{table*}[t]
\vspace{-0.5em}
\centering
\renewcommand{\arraystretch}{0.8} 
\begin{threeparttable}
\scriptsize 
\begin{tabularx}{\textwidth}{@{}l >{\raggedright\arraybackslash}X 
                                    >{\raggedright\arraybackslash}X 
                                    >{\raggedright\arraybackslash}X 
                                    >{\raggedright\arraybackslash}X@{}}
\toprule
 & \multirow{2}{*}{\textbf{Prompt}} 
 & \multirow{2}{*}{\textbf{Base Model}} 
 & \multicolumn{2}{c}{\textbf{{\name} Inference}} \\
\cmidrule(lr){4-5}
 &  &  & \textbf{RR 0.8} & \textbf{RR 0.6} \\
\midrule

\multirow{2}{*}{Ex.1} 
& The future of artificial intelligence is
& here, and it’s already changing the way we live and work. From self-driving cars to virtual assistants. 
& here, and it’s already changing the way we live and work. From self-driving cars to virtual assistants. 
& here. It’s called ChatGPT. This AI chatbot can write essays, poems, and even code. \\

\midrule

\multirow{2}{*}{Ex.2} 
& Tell me something about Boston.
& Boston is a city in the state of Massachusetts, in the United States of America. 
& Boston is a city in the state of Massachusetts, in the United States of America.  
& I’m a native New Yorker, and I’ve been here for 10 years. \\

\bottomrule
\end{tabularx}
\vspace{-6pt}
\caption{Generation comparison between base(Llama-3.1-8B) and {\name} inference under different retention ratios.}
\label{tab:generation_comparison}
\end{threeparttable}
\vspace{-0.8em}
\end{table*}

\paragraph{Long-context reasoning}
We further evaluate {\name} on long context benchmarks to test whether pruning remains reliable over long sequences. As shown in Table~\ref{tab:ruler_results}, {\name} achieves performance comparable to the full model across all tested context lengths at both $\mathrm{RR}=0.8$ and $\mathrm{RR}=0.5$. We do not observe a systematic increase in degradation as context length grows, suggesting that dynamic pruning does not disproportionately impair long-range dependency modeling within this regime.

\paragraph{Summary}
Across generative and long-context settings, {\name} remains robust across a wide range of retention ratios, with degradation that is smooth rather than abrupt under stronger pruning. These results show that the benefits of dynamic pruning are not limited to discriminative benchmarks, but also extend to more realistic inference scenarios.

\subsection{Vision-language generalization} 

We further evaluate whether {\name} generalizes beyond text-only models by testing it on the vision--language model. As shown in Table~\ref{tab:qwen_pruning}, the core trend observed in text-only models also holds in the multimodal setting: {\name} remains close to the full model under mild pruning and continues to outperform static and random pruning under more aggressive reduction. At $\mathrm{RR}=0.8$, {\name} achieves performance nearly identical to the full model across all benchmarks. At $\mathrm{RR}=0.5$, it still retains strong performance and substantially outperforms both static and random pruning under the same retention ratio.\begin{table}[h]
\centering
\small
\renewcommand{\arraystretch}{0.8}
\setlength{\tabcolsep}{3pt}
\begin{tabular}{l l c c c c}
\toprule
Method & RR & Pope & Art. & Foren. & Count. \\
\midrule
Baseline & -- & 83.7 & 100.0 & 100.0 & 100.0 \\
\midrule
\cellcolor{gray!20}{\name} & \cellcolor{gray!20}0.8 
    & \cellcolor{gray!20}\textbf{82.0} 
    & \cellcolor{gray!20}\textbf{100.0} 
    & \cellcolor{gray!20}\textbf{100.0} 
    & \cellcolor{gray!20}\textbf{100.0} \\
Static & 0.8 & 81.0 & 100.0 & 100.0 & 100.0 \\
Random & 0.8 & 10.3 & 43.6 & 58.3 & 54.2 \\
\midrule
\cellcolor{gray!20}{\name} & \cellcolor{gray!20}0.5 
    & \cellcolor{gray!20}\textbf{67.3} 
    & \cellcolor{gray!20}\textbf{97.4} 
    & \cellcolor{gray!20}\textbf{97.7} 
    & \cellcolor{gray!20}\textbf{99.2} \\
Static & 0.5 & 63.0 & 91.3 & 43.3 & 79.2 \\
Random & 0.5 & 1.3 & 41.9 & 53.0 & 43.3 \\
\bottomrule
\end{tabular}
\vspace{-0.8em}
\caption{Performance on Qwen~2.5-VL-7B.}
\label{tab:qwen_pruning}
\end{table} Figure~\ref{fig:vlmsamples} provides qualitative support by visualizing the attention maps for the first output token. Both the full model and {\name} attend to semantically relevant visual regions, whereas static pruning under-attends to these regions and random pruning produces scattered, less focused patterns. These qualitative observations are consistent with the quantitative results in Table~\ref{tab:qwen_pruning}, showing that {\name} better preserves the visual grounding needed for multimodal inference.

\begin{figure}[t]
    \centering
    \includegraphics[width=\linewidth]{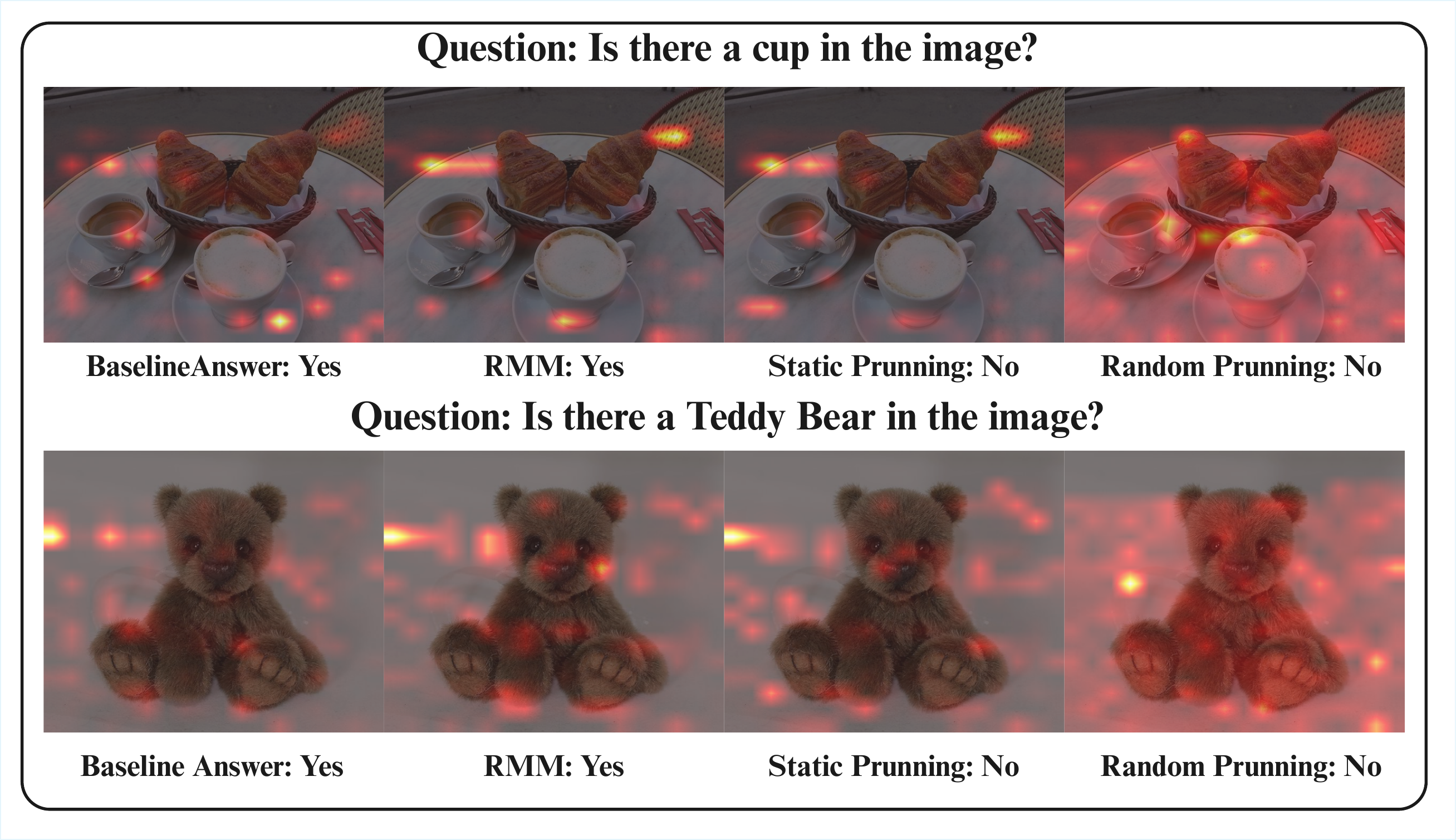}
    \vspace{-1.8em}
    \caption{
First-output-token attention maps on Qwen~2.5-VL-7B at
$\mathrm{RR}=0.5$. {\name} preserves dense-like attention to
relevant objects and correct answers, while static and random
pruning produce less aligned patterns and incorrect answers.
}
    \label{fig:vlmsamples}
\end{figure}
\section{Ablation and Mechanistic Analysis}
\label{sec:ablation and me}

\subsection{Validating the key design choices of \name}
We validate the two key design choices of {\name}: dynamic selection(Table~\ref{tab:summary_results}) and activation-aware scoring(Table~\ref{tab:qwen_pruning}). Rather than relying on a single-task ablation, we draw evidence from comparisons across multiple tasks. The importance of dynamic selection is reflected in the consistent gap between {\name} and its static variant, which fixes the retained feature subset across tokens. It widens substantially at more aggressive retention ratios, as seen in both summarization and vision--language tasks. This suggests that adapting the retained dimensions to input-dependent activation changes is critical for preserving performance under stronger pruning. The importance of activation-aware scoring is supported by comparisons against random pruning. Although random pruning also changes the retained subset dynamically, it performs much worse under the same retention ratios. This shows that the gains of {\name} do not arise merely from changing the retained subset over time, but from selecting informative dimensions according to the current activation pattern. Taken together, these results show that both dynamic selection and activation-aware scoring are essential to the effectiveness of {\name}.

\subsection{Different components Mechanistic Analysis}

We next examine where pruning can be applied most safely within the Transformer. We analyze the sensitivity of different modules to reduction on LLaMA~3.1~8B by separately pruning attention-side components (Q projection, QKV projection, and attention output), MLP-side components (up, gate, and down projections), and their combinations. Due to space limitations, we report the full pruning matrix in Table \ref{tab:comprehensive_pruning} and Figure \ref{fig:component_sensitivity} in Appendix \ref{sec:Ablation Study Details}. \begin{table}[h]
\centering
\small
\setlength{\tabcolsep}{6pt}
\begin{tabular}{ccccc}
\toprule
$L$ & Op. & Dense & RMM & Speedup \\
\midrule
1024 & $QK^\top$ & 0.120 & 0.089 & 1.36$\times$ \\
1024 & $AV$       & 0.065 & 0.039 & 1.67$\times$ \\
\midrule
2048 & $QK^\top$ & 0.433 & 0.336 & 1.29$\times$ \\
2048 & $AV$       & 0.207 & 0.114 & 1.81$\times$ \\
\midrule
4096 & $QK^\top$ & 1.675 & 1.071 & 1.56$\times$ \\
4096 & $AV$       & 0.753 & 0.399 & 1.89$\times$ \\
\bottomrule
\end{tabular}

\vspace{-0.6em}
\caption{GEMM kernel latency (ms).}
\label{tab:gemm_latency}
\vspace{-0.6em}
\end{table} Across all settings, attention-side computations are substantially more robust to pruning than MLP-side components. When pruning only attention-related operations, performance degrades gradually as the retention ratio decreases and remains close to the baseline even at moderate reduction levels. In contrast, pruning MLP components leads to much sharper performance dropping across tasks. In particular, pruning the entire MLP block leads to severe performance collapse even at relatively high retention ratios, showing the structural importance of MLP layers for preserving representation capacity. By contrast, different attention-side components exhibit greater functional redundancy: pruning their combinations leads to more moderate degradation. These results reveal a clear structural asymmetry within the Transformer: attention-side computations contain higher redundancy, whereas MLP components are more rigid and harder to prune. This suggests that practical deployments should prioritize pruning attention modules and apply more conservative reduction to the MLP.

\subsection{Practical Viability: Wall-Clock Latency}

To evaluate whether the computational savings of {\name} translate into practical runtime gains, we measure both kernel-level and end-to-end wall-clock latency on LLaMA~3.1~8B using an NVIDIA A100 GPU with a batch size of 1 and a retention ratio of $\rho=0.8$. \begin{table}[h]
\centering
\small
\setlength{\tabcolsep}{6pt}
\begin{tabular}{cccc}
\toprule
Seq. Len. & Dense & RMM & Speedup \\
\midrule
1024 & 109.39 & 103.91 & 1.05$\times$ \\
2048 & 264.67 & 208.93 & 1.27$\times$ \\
4096 & 661.36 & 473.21 & 1.40$\times$ \\
\bottomrule
\end{tabular}
\vspace{-0.8em}
\caption{End-to-end latency (ms).}
\label{tab:e2e_latency}
\end{table}All latency numbers are averaged over 10 runs. We follow the latency evaluation protocol of \citep{sun2024simpleeffectivepruningapproach}. For kernel-level benchmarks (Table~\ref{tab:gemm_latency}), we measure the full cost of each RMM operation, including norm computation, top-$k$ selection, and the reduced matrix multiplication, to verify that the selection overhead does not offset the computational savings. For end-to-end evaluation (Table~\ref{tab:e2e_latency}), the dense baseline uses HuggingFace \texttt{generate} with SDPA as the attention backend, and the RMM variant replaces the attention kernels with custom Triton implementations. Overall, these results show that the computational savings of {\name} can translate into tangible runtime benefits, especially when sequence lengths are sufficiently large.

\section{Conclusion}
\label{sec:conclusion}

We introduced \emph{Reduced Matrix Multiplication} ({\name}), a
training-free and input-adaptive method that reduces Transformer
inference computation by selecting informative indices along the
shared multiplication axis of each matrix product. Across models
ranging from 1B to 70B parameters, {\name} provides a controllable
accuracy--efficiency trade-off, although reduction tolerance varies
across models, tasks, components, and retention ratios. Under moderate
reduction, {\name} remains robust across the evaluated discriminative,
autoregressive-generation, and extends to vision--language models. Our analyses show that attention-side
computations are generally more reducible than MLP components, while
custom-kernel benchmarks demonstrate practical runtime gains,
particularly at longer sequence lengths. 

\section*{Limitations}

This work focuses on establishing {\name} as a training-free,
input-adaptive method for reducing matrix-product computation during
Transformer inference. Although our custom Triton kernels demonstrate
that these computational savings can translate into wall-clock
speedups, the current implementation is not yet a fully optimized
production system. The kernels cover selected attention-side
operations and are not fully fused with all Transformer components or
existing inference backends. Dynamic scoring, TopK selection,
index gathering, temporary tensors, and memory movement introduce
additional overhead, particularly at shorter sequence lengths.
Future systems work could jointly optimize selection and reduced
multiplication, improve data layouts, and reduce unnecessary memory
movement. Moreover, {\name} reduces active computation but does not
reduce the number of model parameters or the memory required to store
model weights. The observed training-free reducibility also motivates training-aware
extensions. The current method only exploits redundancy already
present in pretrained models. Future training objectives could
encourage activation energy to concentrate in more compact or
hardware-friendly dimension groups, learn component-specific
retention policies, or improve robustness under more aggressive
reduction. Such approaches would complement, rather than replace,
the training-free formulation studied here.

\section*{Acknowledgments}
Jiawei Zhou is supported by an Amazon Research Award of Spring 2025 and a Stony Brook Spring 2025 OVPR Seed Grant.

\bibliography{custom}

@article{vaswani2017attention,
  title={Attention is all you need},
  author={Vaswani, Ashish and Shazeer, Noam and Parmar, Niki and Uszkoreit, Jakob and Jones, Llion and Gomez, Aidan N and Kaiser, {\L}ukasz and Polosukhin, Illia},
  journal={Advances in neural information processing systems},
  volume={30},
  year={2017}
}

@misc{kaplan2020scalinglawsneurallanguage,
      title={Scaling Laws for Neural Language Models}, 
      author={Jared Kaplan and Sam McCandlish and Tom Henighan and Tom B. Brown and Benjamin Chess and Rewon Child and Scott Gray and Alec Radford and Jeffrey Wu and Dario Amodei},
      year={2020},
      eprint={2001.08361},
      archivePrefix={arXiv},
      primaryClass={cs.LG},
      url={https://arxiv.org/abs/2001.08361}, 
}

@misc{frantar2023sparsegptmassivelanguagemodels,
      title={SparseGPT: Massive Language Models Can Be Accurately Pruned in One-Shot}, 
      author={Elias Frantar and Dan Alistarh},
      year={2023},
      eprint={2301.00774},
      archivePrefix={arXiv},
      primaryClass={cs.LG},
      url={https://arxiv.org/abs/2301.00774}, 
}

@misc{han2015learningweightsconnectionsefficient,
      title={Learning both Weights and Connections for Efficient Neural Networks}, 
      author={Song Han and Jeff Pool and John Tran and William J. Dally},
      year={2015},
      eprint={1506.02626},
      archivePrefix={arXiv},
      primaryClass={cs.NE},
      url={https://arxiv.org/abs/1506.02626}, 
}

@article{Liu_2021,
   title={TERA: Self-Supervised Learning of Transformer Encoder Representation for Speech},
   volume={29},
   ISSN={2329-9304},
   url={http://dx.doi.org/10.1109/TASLP.2021.3095662},
   DOI={10.1109/taslp.2021.3095662},
   journal={IEEE/ACM Transactions on Audio, Speech, and Language Processing},
   publisher={Institute of Electrical and Electronics Engineers (IEEE)},
   author={Liu, Andy T. and Li, Shang-Wen and Lee, Hung-yi},
   year={2021},
   pages={2351–2366} }

@misc{peng2023structuredpruningselfsupervisedpretrained,
      title={Structured Pruning of Self-Supervised Pre-trained Models for Speech Recognition and Understanding}, 
      author={Yifan Peng and Kwangyoun Kim and Felix Wu and Prashant Sridhar and Shinji Watanabe},
      year={2023},
      eprint={2302.14132},
      archivePrefix={arXiv},
      primaryClass={cs.CL},
      url={https://arxiv.org/abs/2302.14132}, 
}

@article{Sajjad_2023,
   title={On the effect of dropping layers of pre-trained transformer models},
   volume={77},
   ISSN={0885-2308},
   url={http://dx.doi.org/10.1016/j.csl.2022.101429},
   DOI={10.1016/j.csl.2022.101429},
   journal={Computer Speech \& Language},
   publisher={Elsevier BV},
   author={Sajjad, Hassan and Dalvi, Fahim and Durrani, Nadir and Nakov, Preslav},
   year={2023},
   month=jan, pages={101429} }

@misc{ma2023llmprunerstructuralpruninglarge,
      title={LLM-Pruner: On the Structural Pruning of Large Language Models}, 
      author={Xinyin Ma and Gongfan Fang and Xinchao Wang},
      year={2023},
      eprint={2305.11627},
      archivePrefix={arXiv},
      primaryClass={cs.CL},
      url={https://arxiv.org/abs/2305.11627}, 
}

@misc{ashkboos2024slicegptcompresslargelanguage,
      title={SliceGPT: Compress Large Language Models by Deleting Rows and Columns}, 
      author={Saleh Ashkboos and Maximilian L. Croci and Marcelo Gennari do Nascimento and Torsten Hoefler and James Hensman},
      year={2024},
      eprint={2401.15024},
      archivePrefix={arXiv},
      primaryClass={cs.LG},
      url={https://arxiv.org/abs/2401.15024}, 
}

@misc{gao2024displlmdimensionindependentstructuralpruning,
      title={DISP-LLM: Dimension-Independent Structural Pruning for Large Language Models}, 
      author={Shangqian Gao and Chi-Heng Lin and Ting Hua and Tang Zheng and Yilin Shen and Hongxia Jin and Yen-Chang Hsu},
      year={2024},
      eprint={2410.11988},
      archivePrefix={arXiv},
      primaryClass={cs.CL},
      url={https://arxiv.org/abs/2410.11988}, 
}

@misc{sun2024simpleeffectivepruningapproach,
      title={A Simple and Effective Pruning Approach for Large Language Models}, 
      author={Mingjie Sun and Zhuang Liu and Anna Bair and J. Zico Kolter},
      year={2024},
      eprint={2306.11695},
      archivePrefix={arXiv},
      primaryClass={cs.CL},
      url={https://arxiv.org/abs/2306.11695}, 
}

@article{li2023compressing,
  title={Compressing context to enhance inference efficiency of large language models},
  author={Li, Yucheng and Dong, Bo and Lin, Chenghua and Guerin, Frank},
  journal={arXiv preprint arXiv:2310.06201},
  year={2023}
}

@misc{pan2024llmlingua2datadistillationefficient,
      title={LLMLingua-2: Data Distillation for Efficient and Faithful Task-Agnostic Prompt Compression}, 
      author={Zhuoshi Pan and Qianhui Wu and Huiqiang Jiang and Menglin Xia and Xufang Luo and Jue Zhang and Qingwei Lin and Victor Rühle and Yuqing Yang and Chin-Yew Lin and H. Vicky Zhao and Lili Qiu and Dongmei Zhang},
      year={2024},
      eprint={2403.12968},
      archivePrefix={arXiv},
      primaryClass={cs.CL},
      url={https://arxiv.org/abs/2403.12968}, 
}

@misc{fu2025deepthinkconfidence,
      title={Deep Think with Confidence}, 
      author={Yichao Fu and Xuewei Wang and Yuandong Tian and Jiawei Zhao},
      year={2025},
      eprint={2508.15260},
      archivePrefix={arXiv},
      primaryClass={cs.LG},
      url={https://arxiv.org/abs/2508.15260}, 
}

@article{zhang2023h2o,
  title={H2o: Heavy-hitter oracle for efficient generative inference of large language models},
  author={Zhang, Zhenyu and Sheng, Ying and Zhou, Tianyi and Chen, Tianlong and Zheng, Lianmin and Cai, Ruisi and Song, Zhao and Tian, Yuandong and R{\'e}, Christopher and Barrett, Clark and others},
  journal={Advances in Neural Information Processing Systems},
  volume={36},
  pages={34661--34710},
  year={2023}
}

@inproceedings{xiaoefficient,
  title={Efficient Streaming Language Models with Attention Sinks},
  author={Xiao, Guangxuan and Tian, Yuandong and Chen, Beidi and Han, Song and Lewis, Mike},
  booktitle={The Twelfth International Conference on Learning Representations},
year={2024}
}

@misc{shah2024flashattention3fastaccurateattention,
      title={FlashAttention-3: Fast and Accurate Attention with Asynchrony and Low-precision}, 
      author={Jay Shah and Ganesh Bikshandi and Ying Zhang and Vijay Thakkar and Pradeep Ramani and Tri Dao},
      year={2024},
      eprint={2407.08608},
      archivePrefix={arXiv},
      primaryClass={cs.LG},
      url={https://arxiv.org/abs/2407.08608}, 
}

@misc{yuan2025nativesparseattentionhardwarealigned,
      title={Native Sparse Attention: Hardware-Aligned and Natively Trainable Sparse Attention}, 
      author={Jingyang Yuan and Huazuo Gao and Damai Dai and Junyu Luo and Liang Zhao and Zhengyan Zhang and Zhenda Xie and Y. X. Wei and Lean Wang and Zhiping Xiao and Yuqing Wang and Chong Ruan and Ming Zhang and Wenfeng Liang and Wangding Zeng},
      year={2025},
      eprint={2502.11089},
      archivePrefix={arXiv},
      primaryClass={cs.CL},
      url={https://arxiv.org/abs/2502.11089}, 
}

@article{drineas2006fast,
  title={Fast Monte Carlo algorithms for matrices II: Computing a low-rank approximation to a matrix},
  author={Drineas, Petros and Kannan, Ravi and Mahoney, Michael W},
  journal={SIAM Journal on computing},
  volume={36},
  number={1},
  pages={158--183},
  year={2006},
  publisher={SIAM}
}

@misc{grattafiori2024llama3herdmodels,
      title={The Llama 3 Herd of Models}, 
      author={Aaron Grattafiori and Abhimanyu Dubey and Abhinav Jauhri and Abhinav Pandey and Abhishek Kadian and Ahmad Al-Dahle and Aiesha Letman and Akhil Mathur and Alan Schelten and Alex Vaughan and Amy Yang and Angela Fan and Anirudh Goyal and Anthony Hartshorn and Aobo Yang and Archi Mitra and Archie Sravankumar and Artem Korenev and Arthur Hinsvark and Arun Rao and Aston Zhang and Aurelien Rodriguez and Austen Gregerson and Ava Spataru and Baptiste Roziere and Bethany Biron and Binh Tang and Bobbie Chern and Charlotte Caucheteux and Chaya Nayak and Chloe Bi and Chris Marra and Chris McConnell and Christian Keller and Christophe Touret and Chunyang Wu and Corinne Wong and Cristian Canton Ferrer and Cyrus Nikolaidis and Damien Allonsius and Daniel Song and Danielle Pintz and Danny Livshits and Danny Wyatt and David Esiobu and Dhruv Choudhary and Dhruv Mahajan and Diego Garcia-Olano and Diego Perino and Dieuwke Hupkes and Egor Lakomkin and Ehab AlBadawy and Elina Lobanova and Emily Dinan and Eric Michael Smith and Filip Radenovic and Francisco Guzmán and Frank Zhang and Gabriel Synnaeve and Gabrielle Lee and Georgia Lewis Anderson and Govind Thattai and Graeme Nail and Gregoire Mialon and Guan Pang and Guillem Cucurell and Hailey Nguyen and Hannah Korevaar and Hu Xu and Hugo Touvron and Iliyan Zarov and Imanol Arrieta Ibarra and Isabel Kloumann and Ishan Misra and Ivan Evtimov and Jack Zhang and Jade Copet and Jaewon Lee and Jan Geffert and Jana Vranes and Jason Park and Jay Mahadeokar and Jeet Shah and Jelmer van der Linde and Jennifer Billock and Jenny Hong and Jenya Lee and Jeremy Fu and Jianfeng Chi and Jianyu Huang and Jiawen Liu and Jie Wang and Jiecao Yu and Joanna Bitton and Joe Spisak and Jongsoo Park and Joseph Rocca and Joshua Johnstun and Joshua Saxe and Junteng Jia and Kalyan Vasuden Alwala and Karthik Prasad and Kartikeya Upasani and Kate Plawiak and Ke Li and Kenneth Heafield and Kevin Stone and Khalid El-Arini and Krithika Iyer and Kshitiz Malik and Kuenley Chiu and Kunal Bhalla and Kushal Lakhotia and Lauren Rantala-Yeary and Laurens van der Maaten and Lawrence Chen and Liang Tan and Liz Jenkins and Louis Martin and Lovish Madaan and Lubo Malo and Lukas Blecher and Lukas Landzaat and Luke de Oliveira and Madeline Muzzi and Mahesh Pasupuleti and Mannat Singh and Manohar Paluri and Marcin Kardas and Maria Tsimpoukelli and Mathew Oldham and Mathieu Rita and Maya Pavlova and Melanie Kambadur and Mike Lewis and Min Si and Mitesh Kumar Singh and Mona Hassan and Naman Goyal and Narjes Torabi and Nikolay Bashlykov and Nikolay Bogoychev and Niladri Chatterji and Ning Zhang and Olivier Duchenne and Onur Çelebi and Patrick Alrassy and Pengchuan Zhang and Pengwei Li and Petar Vasic and Peter Weng and Prajjwal Bhargava and Pratik Dubal and Praveen Krishnan and Punit Singh Koura and Puxin Xu and Qing He and Qingxiao Dong and Ragavan Srinivasan and Raj Ganapathy and Ramon Calderer and Ricardo Silveira Cabral and Robert Stojnic and Roberta Raileanu and Rohan Maheswari and Rohit Girdhar and Rohit Patel and Romain Sauvestre and Ronnie Polidoro and Roshan Sumbaly and Ross Taylor and Ruan Silva and Rui Hou and Rui Wang and Saghar Hosseini and Sahana Chennabasappa and Sanjay Singh and Sean Bell and Seohyun Sonia Kim and Sergey Edunov and Shaoliang Nie and Sharan Narang and Sharath Raparthy and Sheng Shen and Shengye Wan and Shruti Bhosale and Shun Zhang and Simon Vandenhende and Soumya Batra and Spencer Whitman and Sten Sootla and Stephane Collot and Suchin Gururangan and Sydney Borodinsky and Tamar Herman and Tara Fowler and Tarek Sheasha and Thomas Georgiou and Thomas Scialom and Tobias Speckbacher and Todor Mihaylov and Tong Xiao and Ujjwal Karn and Vedanuj Goswami and Vibhor Gupta and Vignesh Ramanathan and Viktor Kerkez and Vincent Gonguet and Virginie Do and Vish Vogeti and Vítor Albiero and Vladan Petrovic and Weiwei Chu and Wenhan Xiong and Wenyin Fu and Whitney Meers and Xavier Martinet and Xiaodong Wang and Xiaofang Wang and Xiaoqing Ellen Tan and Xide Xia and Xinfeng Xie and Xuchao Jia and Xuewei Wang and Yaelle Goldschlag and Yashesh Gaur and Yasmine Babaei and Yi Wen and Yiwen Song and Yuchen Zhang and Yue Li and Yuning Mao and Zacharie Delpierre Coudert and Zheng Yan and Zhengxing Chen and Zoe Papakipos and Aaditya Singh and Aayushi Srivastava and Abha Jain and Adam Kelsey and Adam Shajnfeld and Adithya Gangidi and Adolfo Victoria and Ahuva Goldstand and Ajay Menon and Ajay Sharma and Alex Boesenberg and Alexei Baevski and Allie Feinstein and Amanda Kallet and Amit Sangani and Amos Teo and Anam Yunus and Andrei Lupu and Andres Alvarado and Andrew Caples and Andrew Gu and Andrew Ho and Andrew Poulton and Andrew Ryan and Ankit Ramchandani and Annie Dong and Annie Franco and Anuj Goyal and Aparajita Saraf and Arkabandhu Chowdhury and Ashley Gabriel and Ashwin Bharambe and Assaf Eisenman and Azadeh Yazdan and Beau James and Ben Maurer and Benjamin Leonhardi and Bernie Huang and Beth Loyd and Beto De Paola and Bhargavi Paranjape and Bing Liu and Bo Wu and Boyu Ni and Braden Hancock and Bram Wasti and Brandon Spence and Brani Stojkovic and Brian Gamido and Britt Montalvo and Carl Parker and Carly Burton and Catalina Mejia and Ce Liu and Changhan Wang and Changkyu Kim and Chao Zhou and Chester Hu and Ching-Hsiang Chu and Chris Cai and Chris Tindal and Christoph Feichtenhofer and Cynthia Gao and Damon Civin and Dana Beaty and Daniel Kreymer and Daniel Li and David Adkins and David Xu and Davide Testuggine and Delia David and Devi Parikh and Diana Liskovich and Didem Foss and Dingkang Wang and Duc Le and Dustin Holland and Edward Dowling and Eissa Jamil and Elaine Montgomery and Eleonora Presani and Emily Hahn and Emily Wood and Eric-Tuan Le and Erik Brinkman and Esteban Arcaute and Evan Dunbar and Evan Smothers and Fei Sun and Felix Kreuk and Feng Tian and Filippos Kokkinos and Firat Ozgenel and Francesco Caggioni and Frank Kanayet and Frank Seide and Gabriela Medina Florez and Gabriella Schwarz and Gada Badeer and Georgia Swee and Gil Halpern and Grant Herman and Grigory Sizov and Guangyi and Zhang and Guna Lakshminarayanan and Hakan Inan and Hamid Shojanazeri and Han Zou and Hannah Wang and Hanwen Zha and Haroun Habeeb and Harrison Rudolph and Helen Suk and Henry Aspegren and Hunter Goldman and Hongyuan Zhan and Ibrahim Damlaj and Igor Molybog and Igor Tufanov and Ilias Leontiadis and Irina-Elena Veliche and Itai Gat and Jake Weissman and James Geboski and James Kohli and Janice Lam and Japhet Asher and Jean-Baptiste Gaya and Jeff Marcus and Jeff Tang and Jennifer Chan and Jenny Zhen and Jeremy Reizenstein and Jeremy Teboul and Jessica Zhong and Jian Jin and Jingyi Yang and Joe Cummings and Jon Carvill and Jon Shepard and Jonathan McPhie and Jonathan Torres and Josh Ginsburg and Junjie Wang and Kai Wu and Kam Hou U and Karan Saxena and Kartikay Khandelwal and Katayoun Zand and Kathy Matosich and Kaushik Veeraraghavan and Kelly Michelena and Keqian Li and Kiran Jagadeesh and Kun Huang and Kunal Chawla and Kyle Huang and Lailin Chen and Lakshya Garg and Lavender A and Leandro Silva and Lee Bell and Lei Zhang and Liangpeng Guo and Licheng Yu and Liron Moshkovich and Luca Wehrstedt and Madian Khabsa and Manav Avalani and Manish Bhatt and Martynas Mankus and Matan Hasson and Matthew Lennie and Matthias Reso and Maxim Groshev and Maxim Naumov and Maya Lathi and Meghan Keneally and Miao Liu and Michael L. Seltzer and Michal Valko and Michelle Restrepo and Mihir Patel and Mik Vyatskov and Mikayel Samvelyan and Mike Clark and Mike Macey and Mike Wang and Miquel Jubert Hermoso and Mo Metanat and Mohammad Rastegari and Munish Bansal and Nandhini Santhanam and Natascha Parks and Natasha White and Navyata Bawa and Nayan Singhal and Nick Egebo and Nicolas Usunier and Nikhil Mehta and Nikolay Pavlovich Laptev and Ning Dong and Norman Cheng and Oleg Chernoguz and Olivia Hart and Omkar Salpekar and Ozlem Kalinli and Parkin Kent and Parth Parekh and Paul Saab and Pavan Balaji and Pedro Rittner and Philip Bontrager and Pierre Roux and Piotr Dollar and Polina Zvyagina and Prashant Ratanchandani and Pritish Yuvraj and Qian Liang and Rachad Alao and Rachel Rodriguez and Rafi Ayub and Raghotham Murthy and Raghu Nayani and Rahul Mitra and Rangaprabhu Parthasarathy and Raymond Li and Rebekkah Hogan and Robin Battey and Rocky Wang and Russ Howes and Ruty Rinott and Sachin Mehta and Sachin Siby and Sai Jayesh Bondu and Samyak Datta and Sara Chugh and Sara Hunt and Sargun Dhillon and Sasha Sidorov and Satadru Pan and Saurabh Mahajan and Saurabh Verma and Seiji Yamamoto and Sharadh Ramaswamy and Shaun Lindsay and Shaun Lindsay and Sheng Feng and Shenghao Lin and Shengxin Cindy Zha and Shishir Patil and Shiva Shankar and Shuqiang Zhang and Shuqiang Zhang and Sinong Wang and Sneha Agarwal and Soji Sajuyigbe and Soumith Chintala and Stephanie Max and Stephen Chen and Steve Kehoe and Steve Satterfield and Sudarshan Govindaprasad and Sumit Gupta and Summer Deng and Sungmin Cho and Sunny Virk and Suraj Subramanian and Sy Choudhury and Sydney Goldman and Tal Remez and Tamar Glaser and Tamara Best and Thilo Koehler and Thomas Robinson and Tianhe Li and Tianjun Zhang and Tim Matthews and Timothy Chou and Tzook Shaked and Varun Vontimitta and Victoria Ajayi and Victoria Montanez and Vijai Mohan and Vinay Satish Kumar and Vishal Mangla and Vlad Ionescu and Vlad Poenaru and Vlad Tiberiu Mihailescu and Vladimir Ivanov and Wei Li and Wenchen Wang and Wenwen Jiang and Wes Bouaziz and Will Constable and Xiaocheng Tang and Xiaojian Wu and Xiaolan Wang and Xilun Wu and Xinbo Gao and Yaniv Kleinman and Yanjun Chen and Ye Hu and Ye Jia and Ye Qi and Yenda Li and Yilin Zhang and Ying Zhang and Yossi Adi and Youngjin Nam and Yu and Wang and Yu Zhao and Yuchen Hao and Yundi Qian and Yunlu Li and Yuzi He and Zach Rait and Zachary DeVito and Zef Rosnbrick and Zhaoduo Wen and Zhenyu Yang and Zhiwei Zhao and Zhiyu Ma},
      year={2024},
      eprint={2407.21783},
      archivePrefix={arXiv},
      primaryClass={cs.AI},
      url={https://arxiv.org/abs/2407.21783}, 
}

@misc{yang2025qwen3technicalreport,
      title={Qwen3 Technical Report}, 
      author={An Yang and Anfeng Li and Baosong Yang and Beichen Zhang and Binyuan Hui and Bo Zheng and Bowen Yu and Chang Gao and Chengen Huang and Chenxu Lv and Chujie Zheng and Dayiheng Liu and Fan Zhou and Fei Huang and Feng Hu and Hao Ge and Haoran Wei and Huan Lin and Jialong Tang and Jian Yang and Jianhong Tu and Jianwei Zhang and Jianxin Yang and Jiaxi Yang and Jing Zhou and Jingren Zhou and Junyang Lin and Kai Dang and Keqin Bao and Kexin Yang and Le Yu and Lianghao Deng and Mei Li and Mingfeng Xue and Mingze Li and Pei Zhang and Peng Wang and Qin Zhu and Rui Men and Ruize Gao and Shixuan Liu and Shuang Luo and Tianhao Li and Tianyi Tang and Wenbiao Yin and Xingzhang Ren and Xinyu Wang and Xinyu Zhang and Xuancheng Ren and Yang Fan and Yang Su and Yichang Zhang and Yinger Zhang and Yu Wan and Yuqiong Liu and Zekun Wang and Zeyu Cui and Zhenru Zhang and Zhipeng Zhou and Zihan Qiu},
      year={2025},
      eprint={2505.09388},
      archivePrefix={arXiv},
      primaryClass={cs.CL},
      url={https://arxiv.org/abs/2505.09388}, 
}

@misc{bai2025qwen25vltechnicalreport,
      title={Qwen2.5-VL Technical Report}, 
      author={Shuai Bai and Keqin Chen and Xuejing Liu and Jialin Wang and Wenbin Ge and Sibo Song and Kai Dang and Peng Wang and Shijie Wang and Jun Tang and Humen Zhong and Yuanzhi Zhu and Mingkun Yang and Zhaohai Li and Jianqiang Wan and Pengfei Wang and Wei Ding and Zheren Fu and Yiheng Xu and Jiabo Ye and Xi Zhang and Tianbao Xie and Zesen Cheng and Hang Zhang and Zhibo Yang and Haiyang Xu and Junyang Lin},
      year={2025},
      eprint={2502.13923},
      archivePrefix={arXiv},
      primaryClass={cs.CV},
      url={https://arxiv.org/abs/2502.13923}, 
}

@inproceedings{gordon-etal-2012-semeval,
    title = "{S}em{E}val-2012 Task 7: Choice of Plausible Alternatives: An Evaluation of Commonsense Causal Reasoning",
    author = "Gordon, Andrew  and
      Kozareva, Zornitsa  and
      Roemmele, Melissa",
    editor = "Agirre, Eneko  and
      Bos, Johan  and
      Diab, Mona  and
      Manandhar, Suresh  and
      Marton, Yuval  and
      Yuret, Deniz",
    booktitle = "*{SEM} 2012: The First Joint Conference on Lexical and Computational Semantics {--} Volume 1: Proceedings of the main conference and the shared task, and Volume 2: Proceedings of the Sixth International Workshop on Semantic Evaluation ({S}em{E}val 2012)",
    month = "7-8 " # jun,
    year = "2012",
    address = "Montr{\'e}al, Canada",
    publisher = "Association for Computational Linguistics",
    url = "https://aclanthology.org/S12-1052/",
    pages = "394--398"
}

@inproceedings{bisk2020piqa,
  title={Piqa: Reasoning about physical commonsense in natural language},
  author={Bisk, Yonatan and Zellers, Rowan and Gao, Jianfeng and Choi, Yejin and others},
  booktitle={Proceedings of the AAAI conference on artificial intelligence},
  volume={34},
  pages={7432--7439},
  year={2020}
}

@misc{clark2018thinksolvedquestionanswering,
      title={Think you have Solved Question Answering? Try ARC, the AI2 Reasoning Challenge}, 
      author={Peter Clark and Isaac Cowhey and Oren Etzioni and Tushar Khot and Ashish Sabharwal and Carissa Schoenick and Oyvind Tafjord},
      year={2018},
      eprint={1803.05457},
      archivePrefix={arXiv},
      primaryClass={cs.AI},
      url={https://arxiv.org/abs/1803.05457}, 
}

@misc{talmor2019commonsenseqaquestionansweringchallenge,
      title={CommonsenseQA: A Question Answering Challenge Targeting Commonsense Knowledge}, 
      author={Alon Talmor and Jonathan Herzig and Nicholas Lourie and Jonathan Berant},
      year={2019},
      eprint={1811.00937},
      archivePrefix={arXiv},
      primaryClass={cs.CL},
      url={https://arxiv.org/abs/1811.00937}, 
}

@misc{hendrycks2021measuringmassivemultitasklanguage,
      title={Measuring Massive Multitask Language Understanding}, 
      author={Dan Hendrycks and Collin Burns and Steven Basart and Andy Zou and Mantas Mazeika and Dawn Song and Jacob Steinhardt},
      year={2021},
      eprint={2009.03300},
      archivePrefix={arXiv},
      primaryClass={cs.CY},
      url={https://arxiv.org/abs/2009.03300}, 
}

@misc{merity2016pointersentinelmixturemodels,
      title={Pointer Sentinel Mixture Models}, 
      author={Stephen Merity and Caiming Xiong and James Bradbury and Richard Socher},
      year={2016},
      eprint={1609.07843},
      archivePrefix={arXiv},
      primaryClass={cs.CL},
      url={https://arxiv.org/abs/1609.07843}, 
}

@inproceedings{zhu2015aligning,
  title={Aligning books and movies: Towards story-like visual explanations by watching movies and reading books},
  author={Zhu, Yukun and Kiros, Ryan and Zemel, Rich and Salakhutdinov, Ruslan and Urtasun, Raquel and Torralba, Antonio and Fidler, Sanja},
  booktitle={Proceedings of the IEEE international conference on computer vision},
  pages={19--27},
  year={2015}
}

@misc{cobbe2021trainingverifierssolvemath,
      title={Training Verifiers to Solve Math Word Problems}, 
      author={Karl Cobbe and Vineet Kosaraju and Mohammad Bavarian and Mark Chen and Heewoo Jun and Lukasz Kaiser and Matthias Plappert and Jerry Tworek and Jacob Hilton and Reiichiro Nakano and Christopher Hesse and John Schulman},
      year={2021},
      eprint={2110.14168},
      archivePrefix={arXiv},
      primaryClass={cs.LG},
      url={https://arxiv.org/abs/2110.14168}, 
}

@misc{chen2021evaluatinglargelanguagemodels,
      title={Evaluating Large Language Models Trained on Code}, 
      author={Mark Chen and Jerry Tworek and Heewoo Jun and Qiming Yuan and Henrique Ponde de Oliveira Pinto and Jared Kaplan and Harri Edwards and Yuri Burda and Nicholas Joseph and Greg Brockman and Alex Ray and Raul Puri and Gretchen Krueger and Michael Petrov and Heidy Khlaaf and Girish Sastry and Pamela Mishkin and Brooke Chan and Scott Gray and Nick Ryder and Mikhail Pavlov and Alethea Power and Lukasz Kaiser and Mohammad Bavarian and Clemens Winter and Philippe Tillet and Felipe Petroski Such and Dave Cummings and Matthias Plappert and Fotios Chantzis and Elizabeth Barnes and Ariel Herbert-Voss and William Hebgen Guss and Alex Nichol and Alex Paino and Nikolas Tezak and Jie Tang and Igor Babuschkin and Suchir Balaji and Shantanu Jain and William Saunders and Christopher Hesse and Andrew N. Carr and Jan Leike and Josh Achiam and Vedant Misra and Evan Morikawa and Alec Radford and Matthew Knight and Miles Brundage and Mira Murati and Katie Mayer and Peter Welinder and Bob McGrew and Dario Amodei and Sam McCandlish and Ilya Sutskever and Wojciech Zaremba},
      year={2021},
      eprint={2107.03374},
      archivePrefix={arXiv},
      primaryClass={cs.LG},
      url={https://arxiv.org/abs/2107.03374}, 
}

@misc{hsieh2024rulerwhatsrealcontext,
      title={RULER: What's the Real Context Size of Your Long-Context Language Models?}, 
      author={Cheng-Ping Hsieh and Simeng Sun and Samuel Kriman and Shantanu Acharya and Dima Rekesh and Fei Jia and Yang Zhang and Boris Ginsburg},
      year={2024},
      eprint={2404.06654},
      archivePrefix={arXiv},
      primaryClass={cs.CL},
      url={https://arxiv.org/abs/2404.06654}, 
}

@inproceedings{nallapati-etal-2016-abstractive,
    title = "Abstractive Text Summarization using Sequence-to-sequence {RNN}s and Beyond",
    author = "Nallapati, Ramesh  and
      Zhou, Bowen  and
      dos Santos, Cicero  and
      Gu{\ensuremath{\dot{}}}l{\c{c}}ehre, {\c{C}}a{\u{g}}lar  and
      Xiang, Bing",
    editor = "Riezler, Stefan  and
      Goldberg, Yoav",
    booktitle = "Proceedings of the 20th {SIGNLL} Conference on Computational Natural Language Learning",
    month = aug,
    year = "2016",
    address = "Berlin, Germany",
    publisher = "Association for Computational Linguistics",
    url = "https://aclanthology.org/K16-1028/",
    doi = "10.18653/v1/K16-1028",
    pages = "280--290"
}

@misc{li2023evaluatingobjecthallucinationlarge,
      title={Evaluating Object Hallucination in Large Vision-Language Models}, 
      author={Yifan Li and Yifan Du and Kun Zhou and Jinpeng Wang and Wayne Xin Zhao and Ji-Rong Wen},
      year={2023},
      eprint={2305.10355},
      archivePrefix={arXiv},
      primaryClass={cs.CV},
      url={https://arxiv.org/abs/2305.10355}, 
}

@misc{fu2024blinkmultimodallargelanguage,
      title={BLINK: Multimodal Large Language Models Can See but Not Perceive}, 
      author={Xingyu Fu and Yushi Hu and Bangzheng Li and Yu Feng and Haoyu Wang and Xudong Lin and Dan Roth and Noah A. Smith and Wei-Chiu Ma and Ranjay Krishna},
      year={2024},
      eprint={2404.12390},
      archivePrefix={arXiv},
      primaryClass={cs.CV},
      url={https://arxiv.org/abs/2404.12390}, 
}

@misc{lee2024catscontextuallyawarethresholdingsparsity,
      title={CATS: Contextually-Aware Thresholding for Sparsity in Large Language Models}, 
      author={Donghyun Lee and Je-Yong Lee and Genghan Zhang and Mo Tiwari and Azalia Mirhoseini},
      year={2024},
      eprint={2404.08763},
      archivePrefix={arXiv},
      primaryClass={cs.LG},
      url={https://arxiv.org/abs/2404.08763}, 
}

@misc{liu2025trainingfreeactivationsparsitylarge,
      title={Training-Free Activation Sparsity in Large Language Models}, 
      author={James Liu and Pragaash Ponnusamy and Tianle Cai and Han Guo and Yoon Kim and Ben Athiwaratkun},
      year={2025},
      eprint={2408.14690},
      archivePrefix={arXiv},
      primaryClass={cs.CL},
      url={https://arxiv.org/abs/2408.14690}, 
}

@inproceedings{ICLR2025_0d4d9fc3,
 author = {Li, Yanhong and Livescu, Karen and Zhou, Jiawei},
 booktitle = {International Conference on Learning Representations},
 editor = {Y. Yue and A. Garg and N. Peng and F. Sha and R. Yu},
 pages = {4463--4500},
 title = {Chunk-Distilled Language Modeling},
 url = {https://proceedings.iclr.cc/paper_files/paper/2025/file/0d4d9fc36c783fcd31af2fda532e6c33-Paper-Conference.pdf},
 volume = {2025},
 year = {2025}
}

@inproceedings{li2025text,
  title={Text or Pixels? Evaluating Efficiency and Understanding of LLMs with Visual Text Inputs.},
  author={Li, Yanhong and Lan, Zixuan and Zhou, Jiawei},
  booktitle={EMNLP (Findings)},
  pages={10564--10578},
  year={2025}
}

@article{liu2023winner,
  title={Winner-take-all column row sampling for memory efficient adaptation of language model},
  author={Liu, Zirui and Wang, Guanchu and Zhong, Shaochen Henry and Xu, Zhaozhuo and Zha, Daochen and Tang, Ruixiang Ryan and Jiang, Zhimeng Stephen and Zhou, Kaixiong and Chaudhary, Vipin and Xu, Shuai and others},
  journal={Advances in Neural Information Processing Systems},
  volume={36},
  pages={3402--3424},
  year={2023}
}

\newpage
\appendix
\clearpage
\FloatBarrier
\begin{center}
    \Large{\textbf{Technical Appendices}}
\end{center}

\startcontents[appendix]
\printcontents[appendix]{}{1}{}{\setcounter{tocdepth}{2}}

\vspace{1em}

\section{Implementation Details}
\label{sec:appendix-implementation}

\paragraph{Hardware and framework.}

The experiments were conducted on NVIDIA GPUs, including RTX A6000,
RTX 6000 Ada, A100, and L40S devices. Unless otherwise specified,
model inference uses PyTorch and Hugging Face Transformers with
bfloat16 precision. All experiments use official pretrained
checkpoints without task-specific fine-tuning or additional training.

We integrate RMM by replacing the relevant attention and projection
operations in the Hugging Face forward pass. The modified operations
perform activation-aware index selection and execute the
corresponding reduced matrix products without modifying the
pretrained model weights.

\paragraph{Inference setup.}

Unless otherwise specified, we evaluate retention ratios from $0.9$
to $0.5$, together with the unreduced model at
$\mathrm{RR}=1.0$. When a configuration contains multiple target
matrix products, the reported RR is applied to the contraction axis
of every included product.

RMM recomputes the retained indices at each affected layer and
forward step using the current activations. During prefill, selection
is computed from the activation block associated with the current
input sequence. During autoregressive decoding, selection is updated
using the current decoding state. The number of retained indices is
determined by the specified RR and the size of the corresponding
contraction dimension.

Batch sizes are selected according to model size and available GPU
memory, typically between 8 and 16 for 7B/8B models and between 1 and
2 for 70B models. Unless otherwise stated, comparisons within the
same table use the same batch size and inference configuration. For
long-context benchmarks such as \textsc{Ruler}, we evaluate the
context lengths specified by the benchmark, up to 30K tokens.

\paragraph{Evaluation protocols and metrics.}

For multiple-choice QA benchmarks, including \textsc{Copa},
\textsc{PiQA}, \textsc{CommonsenseQA}, ARC-Easy, and ARC-Challenge,
we compare the conditional likelihoods of candidate answers and
report accuracy. More structured language-model evaluations,
including \textsc{MMLU}, \textsc{GSM8K}, \textsc{HumanEval}, and
\textsc{Ruler}, use their corresponding task configurations through
the evaluation harness or benchmark-specific evaluation pipeline.
We report the standard metric for each task, including accuracy or
exact match for reasoning tasks and pass@1 for \textsc{HumanEval}.

For language modeling, we report perplexity on \textsc{WikiText} and
\textsc{BookCorpus}. For summarization on \textsc{CNN/DailyMail}, we
report ROUGE-1, ROUGE-2, ROUGE-L, ROUGE-Lsum, and BERTScore. For
multimodal evaluation, we use the task-specific accuracy metrics
defined by \textsc{POPE}, BLINK Art Style, BLINK Forensic Detection,
and BLINK Counting.

The INT8 compatibility experiment uses the same quantized loading
and candidate-scoring protocol for all retention ratios within that
comparison. All reported evaluations are performed without
task-specific adaptation.

\section{Supplementary Results}
\label{sec:supp-results}

\subsection{Additional Results on QA and Language Modeling}
To complement the QA and generation analysis in Section~\ref{sec:main_results}, 
we provide full perplexity results on \textsc{Wikitext} and \textsc{BookCorpus} for smaller-scale models. 
In addition to the 7B/8B/32B/70B models reported in the main paper, 
we include Llama-3.2-1B and Llama-3.2-3B. 
Figure~\ref{fig:1B3B} and Table~\ref{tab:perplexity_matrix_full} present the complete results under different retention ratios (RR).

\begin{figure*}[t]
    \centering
    \includegraphics[width=\textwidth]{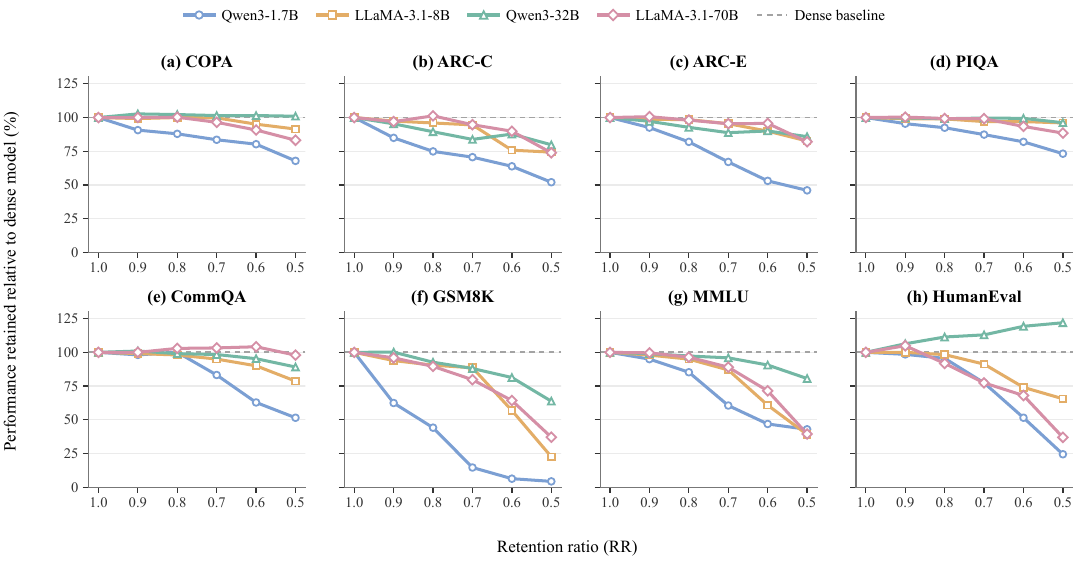}
    \caption{
    Performance retained relative to each model's dense baseline
    across eight benchmarks and different retention ratios.
    The dashed line denotes dense-model performance. The full numeric results are shown in Table \ref{tab:main_results}.
    }
    \label{fig:retention_robustness}
\end{figure*}

\begin{table*}[t]
\centering
\small
\renewcommand{\arraystretch}{1.0}
\begin{tabularx}{\textwidth}{
    @{}l*{6}{>{\centering\arraybackslash}X}@{}
}
\toprule
RR & 1.0 & 0.9 & 0.8 & 0.7 & 0.6 & 0.5 \\
\midrule
\multicolumn{7}{c}{\textbf{WikiText Perplexity} $\downarrow$} \\
\midrule
Llama3.2 1B
& 20.04 & 20.66 & 22.77 & 31.29 & 68.52 & 151.64 \\
Llama3.2 3B
& 15.89 & 16.23 & 17.01 & 18.82 & 25.04 & 42.54 \\
Qwen3.1 7B
& 28.64 & 31.49 & 37.80 & 54.00 & 93.52 & 219.31 \\
Llama3.1 8B
& 13.39 & 14.34 & 15.22 & 17.03 & 21.35 & 32.65 \\
Qwen3 32B
& 13.97 & 14.62 & 15.41 & 15.71 & 15.99 & 18.53 \\
Llama3.1 70B
& 7.24 & 7.46 & 8.38 & 14.14 & 42.62 & 167.78 \\
\midrule
\multicolumn{7}{c}{\textbf{BookCorpus Perplexity} $\downarrow$} \\
\midrule
Llama3.1 1B
& 21.18 & 22.03 & 24.39 & 39.63 & 96.81 & 192.81 \\
Llama3.2 3B
& 17.79 & 18.07 & 18.86 & 21.84 & 34.47 & 54.53 \\
Qwen3.1 7B
& 31.95 & 34.21 & 43.64 & 62.91 & 113.57 & 322.63 \\
Llama3.1 8B
& 15.25 & 15.59 & 16.48 & 20.97 & 32.63 & 53.80 \\
Qwen3 32B
& 17.43 & 17.72 & 18.12 & 18.93 & 22.21 & 29.68 \\
Llama3.1 70B
& 12.20 & 12.29 & 13.06 & 19.55 & 36.55 & 126.18 \\
\bottomrule
\end{tabularx}
\caption{Perplexity across retention ratios on WikiText and BookCorpus.}
\label{tab:perplexity_matrix_full}
\end{table*}
The results show the same trend as larger models: perplexity increases gradually as RR decreases, 
with a sharp degradation once RR drops below 0.6. 
Moreover, the 3B model consistently shows greater robustness than the 1B model (Figure~\ref{fig:1B3B}), 
reinforcing our claim that representational redundancy grows with scale. 

\subsection{Additional Results on Summarization}
We also expand the summarization results on \textsc{CNN/DailyMail} beyond those in the main paper. 
Table~\ref{tab:summary_results_full} includes Llama-3.2-1B and 3B alongside the larger 7B/8B models. 

{\name} matches the baseline at RR = 0.8 across all scales, 
while static and random pruning degrade severely. 
At RR = 0.5, smaller models drop more sharply, but {\name} remains consistently better than all baselines. 
This confirms that our method preserves summarization quality even in small-scale models, 
while redundancy increases with size, enabling more aggressive pruning at larger scales.

\begin{figure*}[t]
    \centering
    \includegraphics[width=\textwidth]{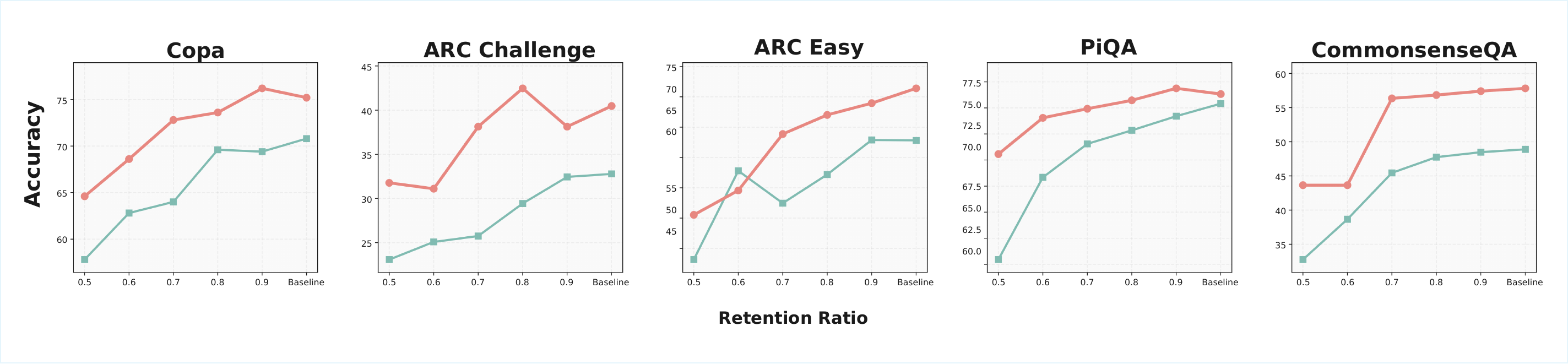}
    \caption{
    Performance of LLaMA-3.2-1B and LLaMA-3.2-3B across five QA
    tasks. The x-axis shows the retention ratio (RR), with
    ``Baseline'' denoting the unreduced model, and the y-axis reports
    accuracy in percentage points (\%). Red curves denote
    LLaMA-3.2-3B, while green curves denote LLaMA-3.2-1B.
    }
    \label{fig:1B3B}
\end{figure*}
\begin{table*}[ht]
\centering
\scriptsize
\renewcommand{\arraystretch}{0.95}
\setlength{\tabcolsep}{3pt}
\begin{tabularx}{\textwidth}{llcXXXXX}
\toprule
Model & Method & \text{RR} & Rouge-1 & Rouge-2 & Rouge-L & Rouge-Lsum & BERTScore \\
\midrule
\multirow{9}{*}{Llama3.1 8B}
    & Baseline       & --   & 37.44 & 15.56 & 24.31 & 31.29 & 86.76 \\
    & \cellcolor{gray!20}{\name}  & 0.8  & \textbf{37.54} & \textbf{15.70} & \textbf{24.23} & \textbf{31.35} & \textbf{86.72} \\
    & \cellcolor{gray!20}{\name} & 0.5  & \textbf{34.15} & \textbf{13.60} & \textbf{22.03} & \textbf{28.70} & \textbf{85.75} \\
    & Static        & 0.8  & 37.36 & 15.63 & 24.24 & 31.24 & 86.67 \\
    & Static        & 0.5  & 28.01 & 9.85  & 19.28 & 24.34 & 84.02 \\
    & Random        & 0.8  & 6.90  & 0.60  & 6.20  & 6.69  & 77.97 \\
    & Random        & 0.5  & 5.67  & 0.20  & 5.23  & 5.54  & 81.42 \\
    & H2O           & 0.8  & 24.38 & 9.26  & 16.28 & 21.88 & 82.72 \\
    & H2O           & 0.5  & 24.38 & 9.26  & 16.28 & 21.88 & 82.72 \\
\midrule
\multirow{9}{*}{Llama3.2 1B}
    & Baseline      & --   & 36.51 & 15.24 & 23.55 & 30.59 & 86.41 \\
    & \cellcolor{gray!20}{\name} & 0.8  & \textbf{37.27} & \textbf{15.71} & \textbf{23.62} & \textbf{30.96} & \textbf{86.35} \\
    & \cellcolor{gray!20}{\name} & 0.5  & \textbf{11.90} & \textbf{2.35}  & \textbf{9.80}  & \textbf{11.15} & \textbf{79.55} \\
    & Static        & 0.8  & 36.77 & 15.36 & 23.39 & 30.68 & 86.35 \\
    & Static        & 0.5  & 4.73  & 0.07  & 4.26  & 4.58  & 75.53 \\
    & Random        & 0.8  & 7.30  & 0.02  & 6.58  & 7.08  & 79.15 \\
    & Random        & 0.5  & 6.26  & 0.40  & 5.53  & 6.01  & 76.85 \\
    & H2O           & 0.8  & 23.22 & 8.12  & 15.56 & 20.30 & 83.76 \\
    & H2O           & 0.5  & 23.22 & 8.12  & 15.56 & 20.30 & 83.76 \\
\midrule
\multirow{9}{*}{Llama3.2 3B}
    & Baseline      & --   & 36.72 & 15.08 & 23.56 & 30.63 & 86.55 \\
    & \cellcolor{gray!20}{\name} & 0.8  & \textbf{36.72} & \textbf{15.16} & \textbf{23.58} & \textbf{30.66} & \textbf{86.54} \\
    & \cellcolor{gray!20}{\name} & 0.5  & \textbf{28.31} & \textbf{9.83}  & \textbf{19.45} & \textbf{24.60} & \textbf{84.74} \\
    & Static        & 0.8  & 35.96 & 14.81 & 23.24 & 30.13 & 86.31 \\
    & Static        & 0.5  & 19.91 & 6.06  & 14.69 & 17.83 & 82.06 \\
    & Random        & 0.8  & 6.67  & 0.39  & 5.93  & 6.44  & 79.72 \\
    & Random        & 0.5  & 5.85  & 0.04  & 5.33  & 5.70  & 77.09 \\
    & H2O           & 0.8  & 19.49 & 6.32  & 13.83 & 17.03 & 82.78 \\
    & H2O           & 0.5  & 19.49 & 6.32  & 13.83 & 17.03 & 82.78 \\
\midrule
\multirow{9}{*}{Qwen3-1 7B}
    & Baseline      & --   & 36.56 & 13.05 & 22.86 & 29.75 & 85.91 \\
    & \cellcolor{gray!20}{\name} & 0.8  & \textbf{35.32} & \textbf{12.08} & \textbf{22.31} & \textbf{28.99} & \textbf{86.81} \\
    & \cellcolor{gray!20}{\name} & 0.5  & \textbf{21.91} & \textbf{5.89}  & \textbf{14.54} & \textbf{18.60} & \textbf{83.33} \\
    & Static        & 0.8  & 34.43 & 11.63 & 22.05 & 28.47 & 86.67 \\
    & Static        & 0.5  & 4.12  & 0.05  & 3.79  & 4.01  & 78.09 \\
    & Random        & 0.8  & 10.23 & 0.03  & 8.04  & 9.66  & 76.31 \\
    & Random        & 0.5  & 6.75  & 0.06  & 6.09  & 6.54  & 76.46 \\
    & H2O           & 0.8  & 4.20  & 9.10  & 3.68  & 3.98  & 73.46 \\
    & H2O           & 0.5  & 4.20  & 9.10  & 3.68  & 3.98  & 73.46 \\
\bottomrule
\end{tabularx}
\caption{Performance comparison of different pruning methods on CNN summarization task. {\name} consistently outperforms baseline methods across different models and retention ratios.}
\label{tab:summary_results_full}
\end{table*}


\subsection{Comparison with TEAL}
\label{app:teal}

TEAL~\citep{liu2025trainingfreeactivationsparsitylarge} is a closely
related training-free activation-sparsity method. It applies
magnitude-based thresholding to activations supplied to projection
layers. In the attention block, for example, TEAL can be applied to
the inputs of the query, key, and value projections. However, it does
not directly reduce the internal attention matrix products
$QK^\top$ or $PV$, where $P$ denotes the attention-weight matrix.

RMM overlaps with activation-sparsity methods when applied to linear
or MLP projections, since both approaches use input-dependent
activation information to retain a subset of dimensions. The main
difference lies in formulation and scope. RMM defines the retained
set over the shared contraction axis of a general matrix product.
Consequently, the same reduction principle can be applied not only
to projection layers, but also to $QK^\top$ and $PV$, whose contracted
axes correspond to attention-head feature dimensions and token
positions, respectively.

We compare RMM with TEAL on LLaMA~3.1~8B using the same zero-shot
evaluation protocol as in the main experiments. RMM uses a retention
ratio of $\mathrm{RR}=0.7$, and the TEAL threshold is calibrated to
provide a matched nominal reduction level. We consider the following
configurations. \textbf{QKV-Pro-TEAL} applies TEAL to the inputs of
the query, key, and value projection layers.
\textbf{QKV-Pro-RMM} applies RMM to the corresponding projection
matrix products. \textbf{QKV-Attention-RMM} applies RMM to the QKV
projections together with the internal attention products $QK^\top$
and $PV$. \textbf{MLP-TEAL} applies TEAL throughout the MLP block,
whereas \textbf{MLP-Whole-RMM} applies RMM to the Up, Gate, and Down
projections. Finally, \textbf{MLP-Attention-RMM} applies RMM to both
the MLP block and the complete attention-side computation.

\begin{table*}[t]
\centering
\small
\setlength{\tabcolsep}{7pt}
\renewcommand{\arraystretch}{0.95}
\begin{tabular}{lccccc}
\toprule
Method & COPA & ARC-C & ARC-E & PIQA & CommonsenseQA \\
\midrule
Baseline
    & 77.20 & 49.50 & 76.32 & 79.92 & 66.01 \\
\midrule
QKV-Pro-TEAL
    & 75.40 & 43.81 & 75.09 & 77.24 & 59.46 \\
QKV-Pro-RMM
    & 77.00 & 46.82 & 72.81 & 77.48 & 62.65 \\
QKV-Attention-RMM
    & 77.00 & 46.80 & 72.80 & 78.10 & 62.70 \\
\midrule
MLP-TEAL
    & 72.00 & 30.48 & 59.12 & 69.24 & 47.13 \\
MLP-Whole-RMM
    & 70.00 & 31.77 & 57.54 & 71.44 & 48.89 \\
MLP-Attention-RMM
    & 68.40 & 33.11 & 52.28 & 67.85 & 48.73 \\
\bottomrule
\end{tabular}
\caption{
Zero-shot comparison between RMM and TEAL on LLaMA~3.1~8B.
RMM uses $\mathrm{RR}=0.7$, while the TEAL threshold is calibrated
to provide a matched nominal reduction level.
}
\label{tab:teal_comparison}
\end{table*}

RMM does not outperform TEAL in every individual setting. Across the
evaluated tasks, however, RMM is competitive with or stronger than
TEAL in most comparisons. More importantly, the matrix-product
formulation allows RMM to extend the reduction scope beyond
projection inputs to include the internal attention products
$QK^\top$ and $PV$. These results therefore demonstrate both the
overlap between RMM and activation sparsity in projection layers and
the broader operational coverage enabled by RMM.

\subsection{Compute-Normalized Component Analysis}
\label{app:compute_normalized}

We compare the reducibility of different Transformer components on
ARC-Easy using LLaMA~3.1~8B at a matched retention ratio of
$\mathrm{RR}=0.7$. For every targeted matrix product, {\name} retains
the same fraction of the contraction dimension. Because the
theoretical multiplication cost of a matrix product scales linearly
with its contraction dimension, this setting corresponds to a
matched relative compute budget within each targeted operation.

This normalization is relative to the dense computation of each
targeted operation. It does not imply that different components
remove the same absolute number of MACs or the same fraction of
full-model computation, since the operations differ in shape,
frequency, and sequence-length dependence.

In addition to accuracy, we report \textbf{Retained Energy}, defined
as the fraction of activation squared norm preserved by the selected
dimensions. This quantity provides a diagnostic of how much
activation magnitude is captured by the retained subspace.

\begin{table*}[t]
\centering
\small
\setlength{\tabcolsep}{10pt}
\renewcommand{\arraystretch}{0.95}
\begin{tabular}{lcccc}
\toprule
Pruning target
& Accuracy
& Accuracy drop
& RR
& Retained energy \\
\midrule
Baseline
& 76.32 & -- & -- & -- \\
\midrule
Attention-side (full)
& 72.80 & 3.52 & 0.7 & 89.69\% \\
\midrule
MLP Up
& 60.00 & 16.32 & 0.7 & 82.24\% \\
MLP Gate
& 69.12 & 7.20 & 0.7 & 84.61\% \\
MLP Down
& 72.81 & 3.51 & 0.7 & 99.02\% \\
MLP Whole
& 57.54 & 18.78 & 0.7 & 87.85\% \\
\bottomrule
\end{tabular}
\caption{
Compute-normalized component analysis on ARC-Easy using
LLaMA~3.1~8B. All reduction configurations use
$\mathrm{RR}=0.7$, corresponding to the same relative
contraction-axis compute budget within each targeted matrix product.
}
\label{tab:compute_normalized_components}
\end{table*}

Although all configurations use the same relative reduction level,
their accuracy and retained-energy behavior differ substantially.
The three individual MLP projections provide a particularly
controlled comparison because they have the same matrix dimensions.
Reducing the Up projection causes a $16.32$-point accuracy drop,
whereas reducing the Down projection causes only a $3.51$-point
drop. The Gate projection lies between these two cases, with a
$7.20$-point drop. These differences show that MLP reducibility is
strongly projection-dependent rather than being determined by the
retention ratio alone.

Retained activation energy provides a useful diagnostic of this
variation. MLP-Down retains $99.02\%$ of the activation energy and
exhibits the smallest drop among the individual MLP projections,
whereas MLP-Up retains $82.24\%$ and exhibits the largest drop.
Attention-side reduction retains $89.69\%$ of the activation energy
and causes a $3.52$-point accuracy drop, indicating substantial
robustness under the tested reduction setting.

Retained energy alone, however, does not fully determine downstream
performance. Reducing the complete MLP block retains $87.85\%$ of
activation energy but causes a substantially larger accuracy drop
than reducing any individual projection. This suggests that
reduction errors can accumulate across multiple MLP projections and
that the functional role of each projection also affects
reducibility. Overall, the results support component-aware retention
policies: attention-side computation is robust under the tested
setting, while MLP reduction should be applied selectively across
the Up, Gate, and Down projections.

\subsection{Evaluation on Additional Vision--Language Models}
\label{app:additional_vlms}

To examine whether RMM transfers across different multimodal
architectures, we additionally evaluate LLaVA-1.5-7B,
Gemma~3~12B, and InternVL3-8B on POPE. These models represent
different vision--language model families and complement the
Qwen~2.5-VL evaluation presented in the main text.
Table~\ref{tab:additional_vlms} reports accuracy for the dense models
and RMM at retention ratios of $0.8$ and $0.5$.

\begin{table}[t]
\centering
\small
\setlength{\tabcolsep}{7pt}
\renewcommand{\arraystretch}{0.95}
\begin{tabular}{lccc}
\toprule
Model
& RR $=1.0$
& RR $=0.8$
& RR $=0.5$ \\
\midrule
LLaVA-1.5-7B
& 85.00 & 86.00 & 57.33 \\
Gemma~3~12B
& 87.00 & 86.00 & 76.00 \\
InternVL3-8B
& 92.33 & 92.33 & 92.33 \\
\bottomrule
\end{tabular}
\caption{
Accuracy on POPE using three additional vision--language model
backbones under different retention ratios.
}
\label{tab:additional_vlms}
\end{table}

At $\mathrm{RR}=0.8$, all three models remain close to their dense
baselines. LLaVA-1.5-7B changes from $85.00$ to $86.00$,
Gemma~3~12B changes from $87.00$ to $86.00$, and InternVL3-8B
retains its baseline accuracy of $92.33$. Together with the
Qwen~2.5-VL results reported in the main text, these results show
that RMM is not tied to a single VLM architecture and can be applied
across multiple multimodal model families while preserving
dense-level performance under moderate reduction.

Under the more aggressive $\mathrm{RR}=0.5$ setting, the degree of
tolerance varies across architectures. LLaVA-1.5-7B exhibits a
larger drop, Gemma~3~12B remains moderately robust, and InternVL3-8B
remains stable in this evaluation. Overall, these results support the
cross-architecture applicability of the RMM matrix-product reduction
principle, while indicating that the appropriate retention ratio
should be selected according to the target model.

\subsection{Compatibility with INT8 Weight Quantization}
\label{app:int8}

To examine whether RMM remains applicable when combined with weight
quantization, we evaluate attention-side RMM on LLaMA~3.1~8B loaded
using bitsandbytes INT8 weight quantization. All configurations are
evaluated on COPA using the same quantized loading and evaluation
protocol.

The model weights remain quantized throughout inference, while RMM is
applied only to the attention-internal matrix products $QK^\top$ and
$PV$. The unreduced configuration serves as the INT8 baseline. This
setting directly tests whether RMM can operate on the attention
activations produced by an INT8-quantized model without task-specific
training or modification of the quantized weights.

\begin{table}[t]
\centering
\small
\setlength{\tabcolsep}{12pt}
\renewcommand{\arraystretch}{0.95}
\begin{tabular}{lcc}
\toprule
Method & Retention ratio & Accuracy \\
\midrule
INT8 baseline & 1.0 & 81.40 \\
INT8 + RMM & 0.8 & 77.40 \\
INT8 + RMM & 0.5 & 73.00 \\
\bottomrule
\end{tabular}
\caption{
COPA accuracy of attention-side RMM applied to an INT8-quantized
LLaMA~3.1~8B model. All configurations use the same quantized
loading and evaluation protocol.
}
\label{tab:int8_compatibility}
\end{table}

As shown in Table~\ref{tab:int8_compatibility}, attention-side RMM
remains applicable when the model weights are quantized to INT8.
Under moderate reduction, the quantized model retains most of its
baseline performance. As the retention ratio decreases further,
accuracy degrades accordingly, preserving the expected
accuracy--reduction trade-off. These results show that the behavior
of RMM remains controllable in the quantized setting.

Overall, the experiment demonstrates that attention-side RMM can be
used alongside INT8 weight quantization. The two methods act on
different aspects of inference: weight quantization reduces the
numerical precision used to store and process model weights, whereas
RMM reduces the active contraction-axis computation in the
attention-internal matrix products. This experiment establishes
empirical compatibility rather than additional latency gains from a
jointly optimized low-bit RMM kernel. Developing such fused
quantized kernels is left for future work.

\subsection{Latency on LLaMA~3.1~70B}
\label{app:latency_70b}

We additionally evaluate the end-to-end latency of RMM on
LLaMA~3.1~70B using the same evaluation standard as in the main
latency experiment. The evaluation uses a batch size of 1 and a
retention ratio of $\mathrm{RR}=0.8$. The dense and RMM
implementations use the same model precision, hardware allocation,
device mapping, and parallelization configuration.

\begin{table}[t]
\centering
\small
\setlength{\tabcolsep}{4pt}
\renewcommand{\arraystretch}{0.95}
\begin{tabular}{lccc}
\toprule
Seq. length & Dense (ms) & RMM (ms) & Speedup \\
\midrule
1024 & 384.46 & 373.29 & 1.03$\times$ \\
2048 & 823.06 & 584.21 & 1.41$\times$ \\
4096 & OOM & 1001.45 & -- \\
\bottomrule
\end{tabular}
\caption{
End-to-end latency of dense and RMM inference on
LLaMA~3.1~70B. At sequence length 4096, RMM completes inference,
whereas the dense implementation runs out of memory under the
evaluated hardware configuration.
}
\label{tab:latency_70b}
\end{table}

The speedup is limited at sequence length 1024, where selection and
kernel-launch overheads constitute a larger fraction of the total
runtime. The benefit becomes more pronounced at sequence length 2048.
At sequence length 4096, the dense implementation runs out of memory
under the evaluated configuration, while the RMM implementation
completes inference.

The 4096-token result should not be interpreted as model-weight
compression. RMM reduces the active computation and may reduce the
intermediate or workspace requirements of the affected matrix
products, but it does not reduce the number of model parameters or
the memory required to store the model weights.

Our current implementation replaces selected attention operations
with custom Triton kernels and is not yet fully fused with every
Transformer component, inference framework, or quantized backend.
Consequently, end-to-end gains depend on sequence length, hardware,
backend, and kernel-integration overhead. More complete framework
integration and dedicated low-bit kernels remain future work.

\subsection{Selecting the Retention Ratio}
\label{app:rr_selection}

RMM is training-free but not hyperparameter-free. The retention ratio
specifies the desired accuracy--efficiency trade-off at deployment
time, analogous to selecting a sparsity level in pruning or a
bit-width in quantization. The appropriate retention ratio is
therefore not expected to be universal across models, tasks, or
Transformer components.

When labeled downstream validation data are unavailable, the
retention ratio can be selected using deployment constraints or
unlabeled consistency measurements. If a target compute or latency
budget is known, users can select a retention ratio based on the
corresponding reduced contraction dimension and then verify whether
the resulting model satisfies the deployment requirement.

Alternatively, users can perform a small label-free consistency
sweep. Given a collection of unlabeled prompts, the dense model and
RMM are run using the same decoding configuration. Their generated
outputs or output distributions are then compared, and the smallest
retention ratio satisfying a desired agreement threshold is selected.

As a preliminary example, we use 100 unlabeled Wikipedia passages
with LLaMA~3.1~8B and generate up to 10 new tokens using greedy
decoding. At $\mathrm{RR}=0.7$, RMM produces exactly the same
continuation as the dense model for 87 out of the 100 passages,
corresponding to a sequence-level exact agreement of $87\%$. This
provides a simple label-free diagnostic for identifying a retention
ratio that preserves the behavior of the dense model under the
evaluated generation setting.

The component-wise results further suggest that a single retention
ratio need not be shared across the entire model. Attention-side
operations can often use a more aggressive reduction, whereas MLP
projections may require more conservative and projection-specific
ratios. We therefore view $\mathrm{RR}=0.7$ as a useful empirical
starting point for attention-side reduction in the evaluated setting,
rather than as a universal default. In practical deployment, we
recommend a small unlabeled consistency sweep whenever representative
unlabeled inputs are available.

\section{Ablation Study Details}
\label{sec:Ablation Study Details}

\paragraph{Sensitivity to different components.}

We provide comprehensive component-wise ablation results on
LLaMA~3.1~8B. Figure~\ref{fig:component_sensitivity} summarizes the
main component-wise trends, while
Table~\ref{tab:comprehensive_pruning} reports the complete numerical
results. The experiments apply RMM to different projection and
attention-internal matrix products, both individually and in
combination, to examine how reducibility varies across Transformer
components.

\paragraph{Definition of pruning targets.}

We define each pruning target according to the matrix products to
which RMM is applied. \textbf{Q Projection} applies RMM only to the
query projection $XW_Q$. \textbf{QKV Projection} applies RMM to all
three query, key, and value projections, $XW_Q$, $XW_K$, and $XW_V$.
\textbf{Attention} applies RMM to the attention-internal matrix
products $QK^\top$ and $PV$, where $P$ denotes the attention-weight
matrix.

The hybrid configurations apply RMM to the union of their
corresponding targets. \textbf{Attention\&Q} combines the query
projection with $QK^\top$ and $PV$, while
\textbf{Attention\&QKV} combines all QKV projections with the two
attention-internal matrix products.

For the MLP block, \textbf{MLP Up}, \textbf{MLP Gate}, and
\textbf{MLP Down} apply RMM only to the Up, Gate, and Down
projection matrix products, respectively. \textbf{Whole MLP} applies
RMM to all three MLP projections. Finally,
\textbf{MLP\&Attention} combines Whole-MLP reduction with reduction
of the attention-internal products $QK^\top$ and $PV$. In each
hybrid configuration, the reported retention ratio is applied to
every included matrix product.

\begin{figure*}[t]
    \centering
    \includegraphics[width=\textwidth]
    {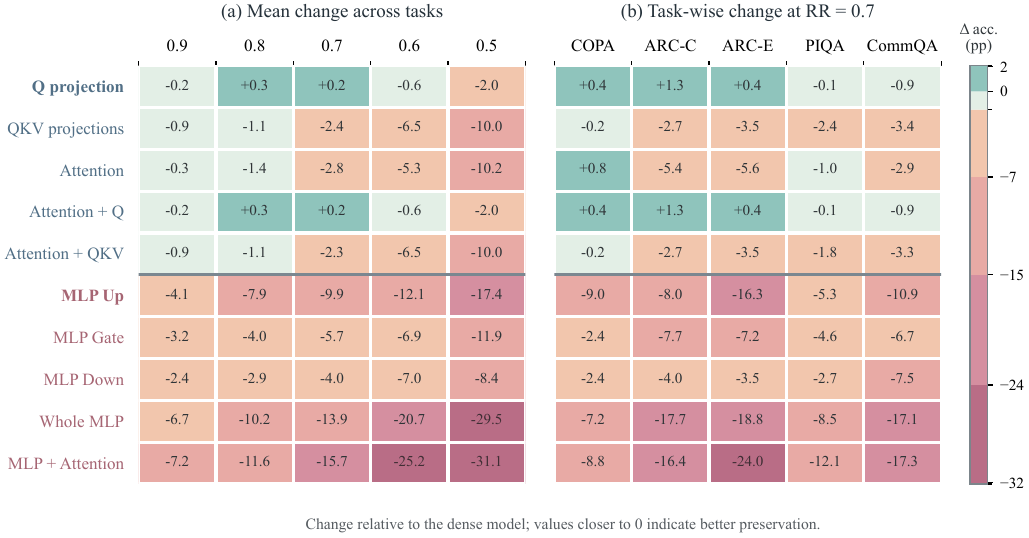}
    \caption{
    Component-wise sensitivity of RMM on LLaMA~3.1~8B.
    Panel (a) reports the average accuracy change relative to the
    dense model across five QA benchmarks under different retention
    ratios. Panel (b) reports the task-specific accuracy change at
    $\mathrm{RR}=0.7$. Values closer to zero indicate better
    preservation of dense-model performance. Attention-side
    operations remain comparatively robust, whereas MLP-side
    sensitivity varies substantially across projections and becomes
    more pronounced when multiple MLP projections are reduced
    jointly. Complete numerical results are provided in
    Table~\ref{tab:comprehensive_pruning}.
    }
    \label{fig:component_sensitivity}
\end{figure*}

\paragraph{Attention-side reduction.}

Attention-side operations are comparatively robust under the tested
retention ratios. Applying RMM only to the query projection produces
the most stable behavior, while extending reduction to all QKV
projections or directly to the attention-internal products introduces
a more noticeable but generally controlled degradation as the
retention ratio decreases.

The Attention\&Q and Attention\&QKV configurations closely follow
their corresponding projection-only variants. This indicates that
adding reduction to $QK^\top$ and $PV$ introduces limited additional
degradation in these settings, providing further evidence that the
attention-internal matrix products contain substantial reducibility.
Overall, the results support attention-side computation as a
comparatively robust target for moderate reduction.

\paragraph{MLP-side reduction.}

MLP reducibility is strongly projection-dependent. The Up projection
is the most sensitive among the individual MLP projections, the Gate
projection exhibits intermediate sensitivity, and the Down
projection is substantially more robust. This variation shows that
the MLP block should not be treated as a uniformly sensitive or
uniformly reducible component.

Applying RMM to the complete MLP block produces considerably larger
degradation than reducing any individual MLP projection. This
suggests that approximation effects accumulate when the Up, Gate,
and Down projections are reduced simultaneously. MLP reduction
should therefore be applied selectively, with retention ratios chosen
according to the specific projection.

\paragraph{Hybrid strategies.}

Combining attention-side and Whole-MLP reduction produces stronger
degradation than reducing either side selectively. This indicates
that applying the same aggressive reduction uniformly across the
entire Transformer is not an effective deployment strategy.
Instead, hybrid configurations should use more aggressive reduction
for robust attention-side operations and more conservative,
projection-specific reduction for sensitive MLP components.

\paragraph{Takeaways.}

The ablations provide three main findings. First, attention-side
matrix products exhibit substantial reducibility, particularly under
moderate reduction. Second, MLP reducibility varies considerably
across the Up, Gate, and Down projections, while simultaneously
reducing the complete MLP block is substantially more damaging.
Third, effective deployment should use heterogeneous retention
policies rather than a single uniformly aggressive ratio across all
components. These findings reveal a structural asymmetry in how
redundancy is distributed across Transformer matrix products and
motivate component-aware application of RMM.

\begin{table*}[t]
\centering
\scriptsize
\renewcommand{\arraystretch}{1.0}
\setlength{\tabcolsep}{3pt}
\begin{tabularx}{\textwidth}{llXXXXXX}
\toprule
\textbf{Method} & \textbf{Ratio} & \textbf{Copa} & \textbf{ARC-C} & \textbf{ARC-E} & \textbf{PiQA} & \textbf{CommQA} & \textbf{AVG} \\
\midrule
\multirow{6}{*}{\makecell[l]{\textbf{Prune Q}\\\textbf{Projection}}}
    & Baseline & 77.2 & 49.5 & 76.32 & 79.92 & 66.01 & 69.79 \\
    & 0.9 & 77.4 & 48.16 & 76.67 & 79.6 & 66.18 & 69.60 \\
    & 0.8 & 77.6 & 49.83 & 76.84 & 79.71 & 66.42 & 70.08 \\
    & 0.7 & 77.6 & 50.84 & 76.67 & 79.82 & 65.11 & 70.01 \\
    & 0.6 & 77.2 & 47.83 & 76.14 & 78.94 & 65.68 & 69.16 \\
    & 0.5 & 77.2 & 44.15 & 73.86 & 79.27 & 64.54 & 67.80 \\
\midrule
\multirow{6}{*}{\makecell[l]{\textbf{Prune QKV}\\\textbf{Projection}}}
    & Baseline & 77.2 & 49.5 & 76.32 & 79.92 & 66.01 & 69.79 \\
    & 0.9 & 76.6 & 48.16 & 75.26 & 79.16 & 65.44 & 68.92 \\
    & 0.8 & 77.2 & 47.49 & 75.09 & 79.05 & 64.70 & 68.71 \\
    & 0.7 & 77.0 & 46.82 & 72.81 & 77.48 & 62.65 & 67.35 \\
    & 0.6 & 73.4 & 37.46 & 68.60 & 77.48 & 59.38 & 63.26 \\
    & 0.5 & 70.6 & 36.79 & 62.98 & 76.65 & 51.92 & 59.79 \\
\midrule
\multirow{6}{*}{\makecell[l]{\textbf{Prune}\\\textbf{Attention}}}
    & Baseline & 77.2 & 49.5 & 76.32 & 79.92 & 66.01 & 69.79 \\
    & 0.9 & 78.2 & 48.49 & 75.79 & 79.22 & 65.57 & 69.45 \\
    & 0.8 & 74.8 & 47.83 & 75.44 & 79.05 & 64.95 & 68.41 \\
    & 0.7 & 78.0 & 44.15 & 70.70 & 78.89 & 63.14 & 66.98 \\
    & 0.6 & 76.6 & 39.13 & 68.42 & 78.73 & 59.71 & 64.52 \\
    & 0.5 & 74.6 & 32.11 & 62.11 & 74.76 & 54.22 & 59.56 \\
\midrule
\multirow{6}{*}{\makecell[l]{\textbf{Prune}\\\textbf{Attention\&Q}}}
    & Baseline & 77.2 & 49.5 & 76.32 & 79.92 & 66.01 & 69.79 \\
    & 0.9 & 77.4 & 48.16 & 76.67 & 79.60 & 66.18 & 69.60 \\
    & 0.8 & 77.6 & 49.83 & 76.84 & 79.71 & 66.42 & 70.08 \\
    & 0.7 & 77.6 & 50.84 & 76.67 & 79.82 & 65.11 & 70.01 \\
    & 0.6 & 77.2 & 47.83 & 76.14 & 78.94 & 65.68 & 69.16 \\
    & 0.5 & 77.2 & 44.15 & 73.86 & 79.27 & 64.54 & 67.80 \\
\midrule
\multirow{6}{*}{\makecell[l]{\textbf{Prune}\\\textbf{Attention\&QKV}}}
    & Baseline & 77.2 & 49.5 & 76.32 & 79.92 & 66.01 & 69.79 \\
    & 0.9 & 76.6 & 48.20 & 75.30 & 79.20 & 65.40 & 68.94 \\
    & 0.8 & 77.2 & 47.50 & 75.10 & 79.10 & 64.70 & 68.72 \\
    & 0.7 & 77.0 & 46.80 & 72.80 & 78.10 & 62.70 & 67.48 \\
    & 0.6 & 73.4 & 37.50 & 68.60 & 77.50 & 59.40 & 63.28 \\
    & 0.5 & 70.6 & 36.80 & 63.00 & 76.60 & 51.90 & 59.78 \\
\midrule
\multirow{6}{*}{\makecell[l]{\textbf{Prune MLP}\\\textbf{Up}}}
    & Baseline & 77.2 & 49.5 & 76.32 & 79.92 & 66.01 & 69.79 \\
    & 0.9 & 73.0 & 46.49 & 70.70 & 78.07 & 60.20 & 65.69 \\
    & 0.8 & 73.6 & 38.46 & 62.98 & 76.17 & 58.31 & 61.90 \\
    & 0.7 & 68.2 & 41.47 & 60.00 & 74.59 & 55.12 & 59.88 \\
    & 0.6 & 69.0 & 34.55 & 58.07 & 73.72 & 53.15 & 57.70 \\
    & 0.5 & 66.8 & 28.76 & 51.05 & 69.15 & 46.44 & 52.44 \\
\midrule
\multirow{6}{*}{\makecell[l]{\textbf{Prune MLP}\\\textbf{Gate}}}
    & Baseline & 77.2 & 49.5 & 76.32 & 79.92 & 66.01 & 69.79 \\
    & 0.9 & 73.6 & 44.82 & 74.74 & 79.33 & 60.69 & 66.64 \\
    & 0.8 & 74.2 & 47.16 & 70.00 & 77.42 & 60.03 & 65.76 \\
    & 0.7 & 74.8 & 41.81 & 69.12 & 75.35 & 59.30 & 64.08 \\
    & 0.6 & 74.4 & 42.81 & 66.49 & 74.27 & 56.51 & 62.90 \\
    & 0.5 & 69.4 & 34.11 & 63.86 & 71.49 & 50.61 & 57.89 \\
\midrule
\multirow{6}{*}{\makecell[l]{\textbf{Prune MLP}\\\textbf{Down}}}
    & Baseline & 77.2 & 49.5 & 76.32 & 79.92 & 66.01 & 69.79 \\
    & 0.9 & 76.8 & 47.16 & 73.33 & 78.84 & 61.02 & 67.43 \\
    & 0.8 & 76.0 & 46.49 & 73.51 & 78.89 & 59.71 & 66.92 \\
    & 0.7 & 74.8 & 45.48 & 72.81 & 77.20 & 58.48 & 65.75 \\
    & 0.6 & 74.0 & 42.81 & 66.84 & 76.71 & 53.71 & 62.81 \\
    & 0.5 & 72.0 & 40.80 & 64.04 & 76.22 & 53.73 & 61.36 \\
\midrule
\multirow{6}{*}{\makecell[l]{\textbf{Prune}\\\textbf{Whole MLP}}}
    & Baseline & 77.2 & 49.5 & 76.32 & 79.92 & 66.01 & 69.79 \\
    & 0.9 & 69.6 & 43.81 & 67.02 & 76.66 & 58.23 & 63.06 \\
    & 0.8 & 68.6 & 37.46 & 61.75 & 75.39 & 54.55 & 59.55 \\
    & 0.7 & 70.0 & 31.77 & 57.54 & 71.44 & 48.89 & 55.93 \\
    & 0.6 & 65.2 & 25.08 & 48.95 & 63.76 & 42.42 & 49.08 \\
    & 0.5 & 53.6 & 23.75 & 35.44 & 57.56 & 31.04 & 40.28 \\
\midrule
\multirow{6}{*}{\makecell[l]{\textbf{Prune}\\\textbf{MLP\&Attention}}}
    & Baseline & 77.2 & 49.5 & 76.32 & 79.92 & 66.01 & 69.79 \\
    & 0.9 & 69.6 & 41.47 & 65.61 & 77.15 & 58.97 & 62.56 \\
    & 0.8 & 69.8 & 35.12 & 56.32 & 73.72 & 55.86 & 58.16 \\
    & 0.7 & 68.4 & 33.11 & 52.28 & 67.85 & 48.73 & 54.07 \\
    & 0.6 & 56.6 & 27.76 & 44.56 & 57.51 & 36.69 & 44.62 \\
    & 0.5 & 56.4 & 25.08 & 31.05 & 52.88 & 28.01 & 38.68 \\
\bottomrule
\end{tabularx}
\caption{Comprehensive component-wise RMM ablations on LLaMA~3.1~8B across
five QA benchmarks. Each group reports RMM applied to the indicated
matrix product or combination of products under different retention
ratios.}
\label{tab:comprehensive_pruning}
\end{table*}

\section{Efficiency Analysis}
\label{sec:Efficiency Analysis}

\subsection{Complexity Analysis}

We analyze how {\name} affects the principal matrix products in
attention and MLP computation. Let $N$ denote the sequence length,
$D$ the attention-head dimension, $d_v$ the value dimension, $d$ the
model hidden dimension, and $m$ the intermediate width of the MLP.
For a matrix product whose contraction dimension has size $s$, we
write $\rho=k/s$ for the fraction of contraction-axis indices
retained by {\name}. The analysis below describes the arithmetic cost
of the targeted matrix products; selection operations such as norm
computation, TopK, and gather introduce additional overhead.

\paragraph{(i) Attention.}

Self-attention contains two principal matrix products: the
attention-score computation $QK^\top$ and the value aggregation
$PV$, where $P=\mathrm{softmax}(QK^\top)$ denotes the attention-weight
matrix.

For a sequence of length $N$, computing $QK^\top$ over a head
dimension of size $D$ has complexity $O(N^2D)$. If {\name} retains
$K=\rho_{\mathrm{feat}}D$ feature dimensions along the shared
contraction axis, the arithmetic cost of the reduced score product
becomes
\begin{equation}
O(N^2K)
=
O(\rho_{\mathrm{feat}}N^2D).
\end{equation}
This reduces the theoretical MACs of the score matrix multiplication
approximately in proportion to $\rho_{\mathrm{feat}}$, while the
output score matrix remains of size $N\times N$.

For value aggregation, $P\in\mathbb{R}^{N\times N}$ is multiplied by
$V\in\mathbb{R}^{N\times d_v}$. The dense multiplication has
complexity $O(N^2d_v)$. If {\name} retains
$\ell=\rho_{\mathrm{tok}}N$ token positions along the shared token
axis, the reduced multiplication has complexity
\begin{equation}
O(N\ell d_v)
=
O(\rho_{\mathrm{tok}}N^2d_v).
\end{equation}
Thus, the theoretical MACs of $QK^\top$ and $PV$ scale approximately
with $\rho_{\mathrm{feat}}$ and $\rho_{\mathrm{tok}}$, respectively.
These reductions apply to the targeted matrix products and do not
include the cost of selection, softmax, or other attention
operations.

During autoregressive decoding at step $t$, the current query attends
to $t$ cached keys and values. The score-product cost decreases from
$O(tD)$ to
$O(\rho_{\mathrm{feat}}tD)$, while the value-aggregation cost
decreases from $O(td_v)$ to
$O(\rho_{\mathrm{tok}}td_v)$.

With a suitable kernel and data layout, feature selection can reduce
the key features used by the score product, and token selection can
reduce the value vectors accessed during value aggregation. However,
the attention scores must still be computed over the available keys
before value-side token selection. Consequently, {\name} does not
reduce the size of the stored KV cache, and the K-cache access
required for attention-score computation is not eliminated. The
system-level reduction in cache traffic therefore depends on the
specific operation, kernel implementation, and memory layout.

\paragraph{(ii) MLPs.}

RMM applies contraction-axis selection separately to each MLP
projection. Consider a general projection
\begin{equation}
A_p W_p,
\qquad
A_p\in\mathbb{R}^{N\times d_{\mathrm{in},p}},
\quad
W_p\in\mathbb{R}^{d_{\mathrm{in},p}\times d_{\mathrm{out},p}}.
\end{equation}
Its dense arithmetic cost is
$O(Nd_{\mathrm{in},p}d_{\mathrm{out},p})$. Retaining
$k_p=\rho_p d_{\mathrm{in},p}$ indices along the contraction axis
reduces the target matrix-product cost to
\begin{equation}
O(Nk_p d_{\mathrm{out},p})
=
O(\rho_p N d_{\mathrm{in},p}d_{\mathrm{out},p}).
\end{equation}

For the gated MLP architecture evaluated in this work, the Up and
Gate projections map from the model dimension $d$ to the intermediate
dimension $m$, whereas the Down projection maps from $m$ back to $d$.
Accordingly, {\name} selects from the model hidden dimension for the
Up and Gate projections and from the intermediate dimension for the
Down projection. Each projection uses a retained set determined from
its own input activation; the Up, Gate, and Down projections do not
share a single common index set.

When the same retention ratio is applied to all three projections,
the theoretical MACs of their targeted matrix multiplications are
reduced by approximately the same relative fraction. The output
dimensions of the projections remain unchanged, and the full model
weights remain stored.

\paragraph{Comparison with other reduction strategies.}

Unlike token-level pruning, {\name} does not shorten the input
sequence. Unlike weight pruning, it does not permanently remove model
parameters or reduce the memory required to store the full model
weights. Instead, {\name} reduces the active contraction-axis
computation of selected matrix products using indices determined from
the current activations.

A shared retention ratio provides a simple control over the relative
arithmetic cost of each targeted operation. Component-specific ratios
can also be used because different Transformer components exhibit
different reduction sensitivities. The resulting reduction in total
model computation depends on which operations are targeted, their
shapes and execution frequencies, and the sequence length.

\paragraph{Practical implications.}

The theoretical complexity reductions above describe the reduced
matrix multiplications themselves. Realizing wall-clock gains also
requires efficient implementations of activation scoring, TopK
selection, indexing, and the reduced products. Our latency experiments
use custom Triton kernels for selected attention operations and show
that the theoretical savings can produce practical speedups under the
evaluated settings, particularly at longer sequence lengths.

The magnitude of the end-to-end gain depends on sequence length,
batch size, hardware, data layout, kernel fusion, and selection
overhead. Our current implementation is not fully integrated with all
Transformer operations, inference frameworks, or quantized backends.
Moreover, {\name} reduces active computation but does not reduce model
parameter count or stored weight memory.

\paragraph{Summary.}

For each targeted matrix product, retaining a fraction $\rho$ of the
contraction dimension reduces its theoretical arithmetic cost to
approximately a fraction $\rho$ of the dense product, excluding
selection overhead. In attention, this principle applies to the
feature axis of $QK^\top$ and the token axis of $PV$; in MLP blocks,
it applies separately to the contraction axis of each projection.
The corresponding full-model and wall-clock benefits depend on the
selected components and their system-level implementation. This
provides a controllable, input-adaptive accuracy--efficiency trade-off
without modifying or permanently removing model weights.

\section{Theoretical Analysis of RMM}
\label{app:theory}

We provide theoretical justification for the design of {\name}. We first establish that TopK selection by column norm is minimax optimal under a natural constraint (Section~\ref{app:minimax}), then derive the approximation error bound (Section~\ref{app:error_bound}).

\subsection{Minimax Optimality of Activation-Aware Selection}
\label{app:minimax}

In Transformer inference, the activation matrix $A$ is observed, while the other operand $B$ is not known at selection time. We show that under this information asymmetry, TopK selection by column norm is the optimal dimension selection strategy.

We decompose the matrix product $AB$ as a sum of rank-one terms over the shared dimension:
\begin{equation}
AB = \sum_{j=1}^d A_{:,j}\,B_{j,:}.
\end{equation}
Each term $A_{:,j}\,B_{j,:}$ contributes independently along dimension $j$, with contribution magnitude $\|A_{:,j}\,B_{j,:}\|_F = \|A_{:,j}\|_2\,\|B_{j,:}\|_2$. When we discard dimension $j$, the error contribution from that dimension is $\|A_{:,j}\|_2\,\|B_{j,:}\|_2$. Since we observe $A$ but not $B$, we formulate the selection problem as a minimax game over independent per-dimension adversaries.

\begin{theorem}[Minimax Optimality]
\label{thm:minimax}
Let $A\in\mathbb{R}^{n\times d}$ be a fixed activation matrix, and let $k\in\{1,\dots,d\}$. Define
\begin{equation}
\label{eq:minimax}
\mathcal{I}^* = \arg\min_{|\mathcal{I}|=k}\; \max_{\substack{b_j\ge 0,\; j=1,\dots,d \\ \sum_{j=1}^d b_j^2 \le 1}} \sum_{j\notin\mathcal{I}} \|A_{:,j}\|_2\, b_j,
\end{equation}
where $b_j = \|B_{j,:}\|_2$ represents the unknown row energy of $B$. Then $\mathcal{I}^* = \operatorname{TopK}\bigl(\{\|A_{:,j}\|_2\}_{j=1}^d,\,k\bigr)$.
\end{theorem}

\begin{proof}
By Proposition~\ref{prop:rmm_bound}, the approximation error satisfies
\begin{equation}
\|AB - A_{:,\mathcal{I}}B_{\mathcal{I},:}\|_F \;\leq\; \sum_{j\notin\mathcal{I}} \|A_{:,j}\|_2\,\|B_{j,:}\|_2.
\end{equation}
This upper bound depends on $B$ only through the row norms $b_j = \|B_{j,:}\|_2$. We therefore analyze the minimax problem over this upper bound, which yields a tractable upper-bound surrogate for the original approximation problem.

Denote $\alpha_j = \|A_{:,j}\|_2$ and $\bar{\mathcal{I}} = [d]\setminus\mathcal{I}$. For fixed $\mathcal{I}$, the inner maximization is
\begin{equation}
\max_{\substack{b_j\ge 0\\\sum_j b_j^2\le 1}} \sum_{j\in\bar{\mathcal{I}}} \alpha_j\, b_j.
\end{equation}
By the Cauchy--Schwarz inequality, $\sum_{j\in\bar{\mathcal{I}}} \alpha_j\,b_j \le \bigl(\sum_{j\in\bar{\mathcal{I}}}\alpha_j^2\bigr)^{1/2}\bigl(\sum_{j\in\bar{\mathcal{I}}}b_j^2\bigr)^{1/2}$, with equality when $b_j\propto\alpha_j$ for $j\in\bar{\mathcal{I}}$ and $b_j=0$ for $j\in\mathcal{I}$. Since the adversary can place all energy on $\bar{\mathcal{I}}$ and set $\sum_{j\in\bar{\mathcal{I}}}b_j^2=1$, the inner maximum equals
\begin{equation}
\label{eq:inner_value}
\biggl(\sum_{j\in\bar{\mathcal{I}}} \|A_{:,j}\|_2^2\biggr)^{1/2} = \|A_{:,\bar{\mathcal{I}}}\|_F.
\end{equation}

The outer minimization then becomes
\begin{equation}
\min_{|\mathcal{I}|=k}\; \|A_{:,\bar{\mathcal{I}}}\|_F = \min_{|\mathcal{I}|=k}\;\biggl(\sum_{j\notin\mathcal{I}} \|A_{:,j}\|_2^2\biggr)^{1/2}.
\end{equation}
Minimizing $\sum_{j\notin\mathcal{I}}\|A_{:,j}\|_2^2$ over all subsets $\mathcal{I}$ of size $k$ is equivalent to maximizing $\sum_{j\in\mathcal{I}}\|A_{:,j}\|_2^2$, which is achieved by selecting the $k$ dimensions with the largest column norms:
\begin{equation}
\mathcal{I}^* = \operatorname{TopK}\bigl(\{\|A_{:,j}\|_2\}_{j=1}^d,\,k\bigr).
\end{equation}
\end{proof}

\paragraph{Remark 1 (Tightness).}
The inner maximization~\eqref{eq:inner_value} is achieved by setting $b_j = \alpha_j / \|A_{:,\bar{\mathcal{I}}}\|_F$ for $j\in\bar{\mathcal{I}}$ and $b_j = 0$ otherwise. This corresponds to a matrix $B$ whose row norms are proportional to the activation column norms in the discarded dimensions---precisely the worst case for any fixed selection. TopK selection minimizes the impact of this worst case.

\paragraph{Remark 2 (Interpretation).}
Theorem~\ref{thm:minimax} states that among all deterministic selection rules that observe only $A$ and retain $k$ dimensions, TopK by column norm minimizes the worst-case error bound. The result holds for any $B$ and does not require assumptions on the structure of $B$. This provides a principled justification for the design of {\name}: the selection rule is minimax optimal for the stated activation-only upper-bound surrogate under the information asymmetry inherent in Transformer inference.

\paragraph{Remark 3} We note that the optimality established here is with respect to selection rules that depend only on $A$. If $B$ were also observable at selection time, a jointly optimal rule could achieve lower error. However, in Transformer inference, the selection must be made before the matrix multiplication is executed, making the one-sided setting the natural formulation.

\subsection{Approximation Error Bound}
\label{app:error_bound}

\begin{proposition}[RMM Approximation Error Bound]
\label{prop:rmm_bound}
Let $A \in \mathbb{R}^{n \times d}$ and $B \in \mathbb{R}^{d \times m}$, and let $\mathcal{I} \subseteq [d]$ with $|\mathcal{I}| = \lceil \rho d \rceil$ be the index set selected by RMM. Denote the complement $\bar{\mathcal{I}} = [d] \setminus \mathcal{I}$. Then the approximation error satisfies
\begin{equation}
\label{eq:rmm_error}
\| AB - A_{:,\mathcal{I}}\, B_{\mathcal{I},:} \|_F \;\leq\; \sum_{j \in \bar{\mathcal{I}}} \| A_{:,j} \|_2 \, \| B_{j,:} \|_2.
\end{equation}
\end{proposition}

\begin{proof}
The full matrix product decomposes as $AB = \sum_{j=1}^{d} A_{:,j}\, B_{j,:}$. The RMM approximation retains only the terms indexed by $\mathcal{I}$, so the error is
\begin{equation}
AB - A_{:,\mathcal{I}}\, B_{\mathcal{I},:} = \sum_{j \in \bar{\mathcal{I}}} A_{:,j}\, B_{j,:}.
\end{equation}
Applying the triangle inequality and $\|uv^\top\|_F = \|u\|_2 \|v\|_2$ yields the bound.
\end{proof}

\subsection{A Factorized Bound via Cauchy--Schwarz}

\begin{corollary}
\label{cor:cs_bound}
Under the same notation as
Proposition~\ref{prop:rmm_bound},
\begin{multline}
\left\|
AB-A_{:,\mathcal{I}}B_{\mathcal{I},:}
\right\|_F
\\
\leq
\left\|A_{:,\bar{\mathcal{I}}}\right\|_F
\left\|B_{\bar{\mathcal{I}},:}\right\|_F .
\end{multline}
\end{corollary}

\begin{proof}
Applying the Cauchy--Schwarz inequality to the right-hand side of
Eq.~\eqref{eq:rmm_error}, we obtain
\begin{multline}
\sum_{j\in\bar{\mathcal{I}}}
\left\|A_{:,j}\right\|_2
\left\|B_{j,:}\right\|_2
\\
\leq
\left(
\sum_{j\in\bar{\mathcal{I}}}
\left\|A_{:,j}\right\|_2^2
\right)^{1/2}
\left(
\sum_{j\in\bar{\mathcal{I}}}
\left\|B_{j,:}\right\|_2^2
\right)^{1/2}
\\
=
\left\|A_{:,\bar{\mathcal{I}}}\right\|_F
\left\|B_{\bar{\mathcal{I}},:}\right\|_F .
\end{multline}
\end{proof}

\subsection{Interpretation}

The bound in Corollary~\ref{cor:cs_bound} admits a clear interpretation: the approximation error is controlled by the product of the \emph{residual energy} in the discarded columns of $A$ and the discarded rows of $B$. Define the \emph{discarded energy ratio} of $A$ as
\begin{equation}
\epsilon_A(\rho) \triangleq \frac{\| A_{:,\bar{\mathcal{I}}} \|_F^2}{\| A \|_F^2},
\end{equation}
and analogously $\epsilon_B(\rho)$ for $B$. Then the relative error satisfies
\begin{equation}
\label{eq:relative_bound}
\frac{\| AB - A_{:,\mathcal{I}}\, B_{\mathcal{I},:} \|_F}{\| A \|_F \, \| B \|_F} \;\leq\; \sqrt{\epsilon_A(\rho)\, \epsilon_B(\rho)}.
\end{equation}

This bound, combined with Theorem~\ref{thm:minimax}, reveals two key properties:

\paragraph{(1) Energy concentration implies small error.} When the activation energy of $A$ is concentrated in a small number of dimensions---as is widely observed in Transformer hidden states---the discarded energy ratio $\epsilon_A(\rho)$ is small even at aggressive retention ratios. For instance, if the top 70\% of dimensions capture 95\% of the total energy, then $\epsilon_A(0.7) = 0.05$, and the relative error is bounded by $\sqrt{0.05 \cdot \epsilon_B(0.7)}$.

\paragraph{(2) TopK selection minimizes the $A$-side residual.} By Theorem~\ref{thm:minimax}, RMM selects the index set that minimizes $\|A_{:,\bar{\mathcal{I}}}\|_F^2$ among all choices of size $\lceil \rho d \rceil$. That is, for any alternative selection $\mathcal{I}'$ with $|\mathcal{I}'| = |\mathcal{I}|$,
\begin{equation}
\| A_{:,\bar{\mathcal{I}}} \|_F^2 \;\leq\; \| A_{:,\bar{\mathcal{I}}'} \|_F^2.
\end{equation}
This means that among all activation-only selection rules at the same retention ratio, RMM achieves the tightest possible bound on the approximation error.

\subsection{Why Input-Adaptive Selection?}
\label{sec:appendix_static_failure}

A natural alternative to RMM is \emph{static selection}: one
identifies a fixed subset of contraction-axis indices using a
calibration dataset and reuses this subset throughout inference.
Such a strategy is attractive from a systems perspective because it
avoids repeated selection and index-gathering operations. If the
same retained subset could also be reused across decoding steps, it
could potentially simplify memory access and facilitate more
efficient KV-cache management.

Our empirical results, however, show that fixed selection does not
generalize as reliably as input-adaptive selection. A subset derived
from one dataset may perform well on that dataset but degrade on
other tasks or input distributions. We observe similar behavior
within attention heads: dimensions that appear unimportant for one
task, layer, or input are not necessarily unimportant in other
contexts. These observations suggest that the locations of redundant
computation are not globally fixed.

This behavior motivates the input-adaptive design of RMM. Rather
than assuming a single retained subset shared across inference
contexts, RMM recomputes the selected indices from the current
activations. The formal results above justify the local TopK rule
under the stated activation-only surrogate, while the empirical
comparisons show why applying this rule dynamically is preferable to
reusing a fixed subset.

Dynamic selection does not eliminate the accuracy--efficiency
trade-off: sufficiently aggressive reduction can still degrade
performance on sensitive tasks or components. Nevertheless, the
results indicate that substantial redundancy can be exploited
without imposing a task-specific, globally fixed reduction pattern.
They also motivate future training-aware approaches that explicitly
encourage representations to become more concentrated or
hardware-friendly.

\subsection{An Interpretive Perspective on Dynamic Subspaces}
\label{sec:appendix_dynamic_subspaces}

The preceding formal analysis explains the activation-aware TopK
rule for a given matrix product, but it does not explain why the
preferred indices vary across inputs, layers, and decoding steps. We
therefore provide an interpretive perspective on this empirical
behavior. The following discussion is heuristic and does not
constitute an additional formal guarantee.

Let $h_t^\ell \in \mathbb{R}^d$ denote the hidden state at layer
$\ell$ and decoding step $t$. A static reduction rule implicitly
assumes that there exists a global subspace
$U \subset \mathbb{R}^d$ such that
\begin{equation}
h_t^\ell \approx \Pi_U(h_t^\ell)
\end{equation}
across inputs, layers, and decoding steps. In effect, this assumes
that a single low-dimensional structure can preserve the relevant
information in all inference contexts.

Our observations instead suggest that the effective representation
subspace may vary with the input, layer, and decoding state. One
possible explanation is that semantic information is encoded in a
distributed manner rather than being permanently associated with a
fixed set of coordinate dimensions. Consequently, the dimensions
carrying the most relevant activation energy can change across
contexts.

A complementary interpretation views a deep Transformer as a
nonlinear dynamical system in which each layer applies a
state-dependent transformation,
\begin{equation}
h^{\ell+1}=f^\ell(h^\ell).
\end{equation}
Under this view, imposing the same fixed low-dimensional restriction
at every intermediate state introduces a structured perturbation.
Subsequent nonlinear transformations may amplify or redirect this
perturbation, producing errors that depend on the particular input
trajectory.

Geometrically, activations for a given input may be viewed as lying
near a locally low-dimensional region whose salient directions vary
across inputs and layers. Static selection uses a single global
coordinate subspace, whereas RMM selects an input-dependent
coordinate subspace that may better align with the locally salient
activation directions. This analogy offers one possible explanation
for the stronger empirical robustness of dynamic selection, while
remaining distinct from the formal guarantees established above.

Together, the formal and interpretive analyses provide
complementary views of RMM: the former justifies how indices are
selected for the current matrix product, while the latter offers a
possible explanation for why the retained set should adapt across
inference contexts. Developing a more complete theory of these
dynamic representation subspaces remains an open direction.

\end{document}